\documentclass[sigconf]{acmart}
\usepackage{amsmath}
\usepackage{amsthm}
\usepackage{booktabs}
\usepackage{graphicx}
\usepackage{multirow}
\usepackage{microtype}
\usepackage[table]{xcolor}
\usepackage{colortbl} 
\usepackage{tabularx}
\newcolumntype{Y}{>{\centering\arraybackslash}X}
\usepackage{bm}
\usepackage{pifont}
\usepackage{xcolor}
\definecolor{darkblue}{RGB}{0, 82, 155}
\usepackage{algorithm}
\usepackage{algpseudocode}
\makeatletter
\patchcmd{\@mkbibcitation}{\ref{TotPages}}{10}{}{
  \PackageError{TMUR}{Cannot preserve published page count}{}
}
\patchcmd{\@mkbibcitation}{\ref{TotPages}}{10}{}{
  \PackageError{TMUR}{Cannot preserve published page count}{}
}
\makeatother
\AtBeginDocument{%
  }

\copyrightyear{2026}
\acmYear{2026}
\setcopyright{cc}
\setcctype{by}
\acmConference[MM '26]{Proceedings of the 34th ACM International Conference on Multimedia}{November 10--14, 2026}{Rio de Janeiro, Brazil}
\acmBooktitle{Proceedings of the 34th ACM International Conference on Multimedia (MM '26), November 10--14, 2026, Rio de Janeiro, Brazil}
\acmDOI{10.1145/3767308.3836405}
\acmISBN{979-8-4007-2213-4/2026/11}

\begin{document}

\title{Are Independently Estimated View Uncertainties Comparable? Unified Routing for Trusted Multi-View Classification}
\renewcommand{\shorttitle}{Unified Routing for Trusted Multi-View Classification}

\author{Yilin Zhang}
\affiliation{%
  \institution{Xidian University}
  \city{Xi'an}
  \country{China}
}
\email{ylzhang\_3@stu.xidian.edu.cn}

\author{Cai Xu}
\affiliation{%
  \institution{Xidian University}
  \city{Xi'an}
  \country{China}
}
\email{cxu@xidian.edu.cn}

\author{Haishun Chen}
\affiliation{%
  \institution{Xidian University}
  \city{Xi'an}
  \country{China}
}
\email{chenhaishun@stu.xidian.edu.cn}

\author{Ziyu Guan}
\affiliation{%
  \institution{Xidian University}
  \city{Xi'an}
  \country{China}
}
\email{zyguan@xidian.edu.cn}

\author{Wei Zhao}
\authornote{Corresponding author.}
\affiliation{%
  \institution{Xidian University}
  \city{Xi'an}
  \country{China}
}
\email{ywzhao@mail.xidian.edu.cn}

\renewcommand{\shortauthors}{Yilin Zhang, Cai Xu, Haishun Chen, Ziyu Guan, \& Wei Zhao}

\begin{abstract}
Trusted multi-view classification typically relies on a view-wise evidential fusion process: each view independently produces class evidence and uncertainty, and the final prediction is obtained by aggregating these independent opinions. While this design is modular and uncertainty-aware, it implicitly assumes that evidence from different views is numerically comparable. In practice, however, this assumption is fragile. Different views often differ in feature space, noise level, and semantic granularity, while independently trained branches are optimized only for prediction correctness, without any constraint enforcing cross-view consistency in evidence strength. As a result, the uncertainty used for fusion can be dominated by branch-specific scale bias rather than true sample-level reliability. To address this issue, we propose \textbf{\underline{T}rusted \underline{M}ulti-view learning with \underline{U}nified \underline{R}outing (TMUR)}, which decouples view-specific evidence extraction from fusion arbitration. TMUR uses view-private experts and one collaborative expert, and employs a unified router that observes the global multi-view context to generate sample-level expert weights. Soft load-balancing and diversity regularization further encourage balanced expert utilization and more discriminative expert specialization. 
We also provide theoretical analysis showing why independent evidential supervision does not identify a common cross-view evidence scale.
Extensive experiments on 14 datasets and comparisons with 15 recent baselines demonstrate that TMUR consistently improves both classification performance and reliability. 
Code is available at \href{https://github.com/YilinZhang107/TMUR}{\textcolor{darkblue}{\textit{here}}}.
\end{abstract}

\ccsdesc[500]{Computing methodologies~Machine learning}
\ccsdesc[500]{Computing methodologies~Supervised learning by classification}

\keywords{Trusted Multi-view Classification, Uncertainty-aware Deep Learning, Evidential Deep Learning.}


\maketitle

\section{Introduction}
Multi-view classification~\cite{lan2025BCM,liu2025enhancingTCSVT} aims to improve prediction by integrating complementary information from multiple views. In many real-world scenarios, different views may exhibit distinct noise levels~\cite{wang2023uncertainty}, semantic granularity, or even conflicting clues~\cite{hu2026robust} for the same sample, causing their quality to vary dynamically across samples. As a result, a reliable multi-view model should not only make correct predictions, but also be able to dynamically assess which views are more trusted for the current sample.

\begin{figure*}[!ht]
\centering
\includegraphics[width=\textwidth]{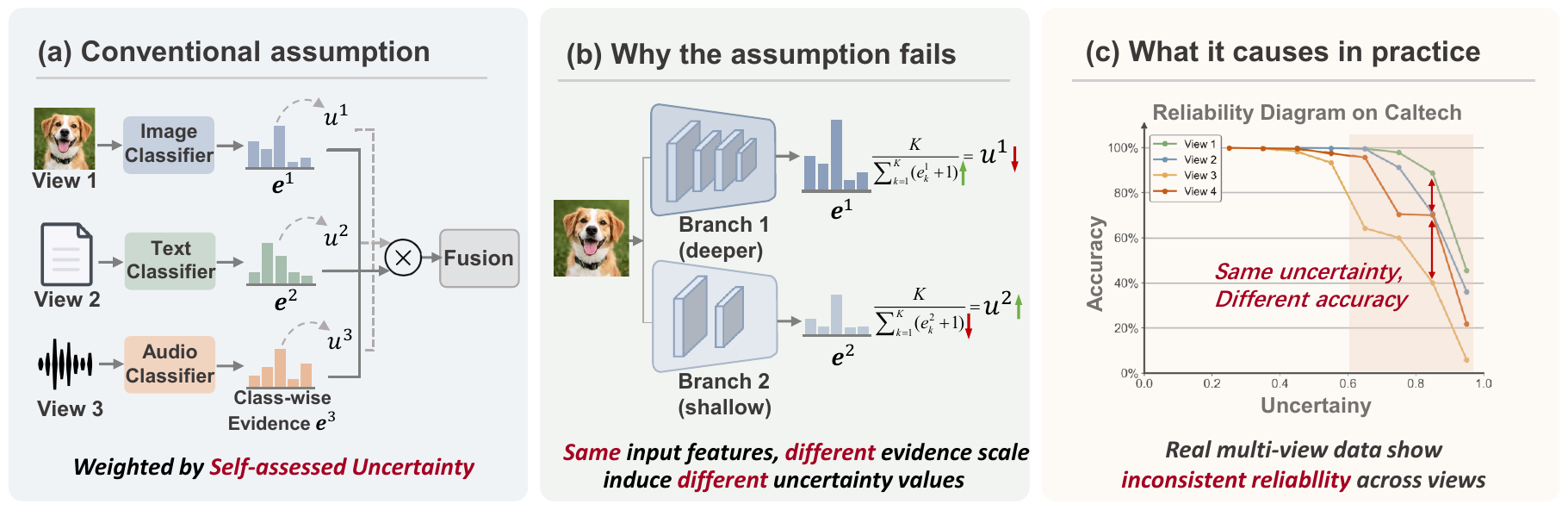}
\caption{(a) Conventional trusted multi-view fusion uses branch-local self-assessed uncertainty for weighting. 
(b) However, heterogeneous branches can produce different evidence scales, so even the same input may yield different uncertainty values. 
(c) On Caltech dataset, different views show inconsistent reliability under the same uncertainty level. 
This suggests that branch-local uncertainty is not directly comparable across views and may be unreliable for cross-view arbitration.}
\Description{Motivation of this work. 
Traditional trusted multi-view fusion uses each branch's self-assessed uncertainty to weight view opinions, implicitly assuming that uncertainty is directly comparable across views. However, this assumption is fragile. Since evidential uncertainty is determined by evidence magnitude, heterogeneous branches can produce different evidence scales even for the same input, yielding inconsistent uncertainty values. Moreover, reliability diagrams on Caltech show that different views can exhibit markedly different accuracies at the same uncertainty level. These observations indicate that branch-local uncertainty is not a reliable cross-view arbitration signal, motivating a unified routing mechanism instead of self-weighted late fusion.}
\label{fig:motivation}
\end{figure*}

To this end, recent trusted multi-view learning methods~\cite{han2023dynamic,xu2024reliable,zhang2026NoisySup,liu2025enhancingAAAI} are largely built on evidential learning~\cite{sensoy2018evidential,deng2023uncertainty} and subjective logic~\cite{josang2018subjective,shafer1976mathematical}. A typical paradigm is shown in Fig.~\ref{fig:motivation}(a), which lets each view-specific branch produce non-negative class evidence, convert it into an uncertainty-aware opinion, and then fuse these opinions according to their branch-wise uncertainty. This line of work is attractive because it is modular, compatible with heterogeneous encoders, and naturally provides an interpretable fusion process through uncertainty-aware weighting. However, it also relies on an assumption that is rarely questioned explicitly: \emph{the evidence and self-assessed uncertainty from different views are treated as directly comparable across branches.}

We find that this assumption is in fact fragile, and it arises from two sources: (1) \textbf{Data-level heterogeneity}. Different views often reside in different feature spaces and exhibit different signal-to-noise ratios~\cite{zhang2024AdapRegress,zhu2026online}, statistical complexity, and levels of semantic abstraction, which naturally induce different evidence distributions. (2) \textbf{Model-level heterogeneity}. Each branch is trained primarily to improve its own classification objective, typically by increasing the evidence assigned to the ground-truth class and suppressing other classes. Such supervision encourages per-branch correctness, but places no explicit constraint on the absolute strength of evidence across views~\cite{xu2025beyond,liu2026beyond}. Moreover, different architectural choices can further amplify this mismatch. 
Therefore, features of the same quality may still yield evidence with significantly different strengths. 
A toy experiment, as shown in Fig.~\ref{fig:motivation}(b), demonstrates that for the same input features, simply changing the depth of one branch can lead to significant differences in evidence strength and induce different uncertainty values.

This issue becomes particularly critical in trusted multi-view fusion, because subjective-logic uncertainty is directly tied to total evidence strength. When the uncertainty used for fusion weighting is dominated by branch-dependent scale rather than true sample-wise reliability, it is meaningful only within each branch and loses fairness as a cross-view arbitration signal, as shown in Fig.~\ref{fig:motivation}(c). Therefore, the challenge is not merely how to estimate uncertainty within each branch, but how to perform \emph{fair cross-view arbitration when evidence scales are not directly comparable}.

This observation calls for a rethinking of the design principle of trusted multi-view fusion. If raw cross-view evidence cannot be compared reliably in magnitude, then fusion weights should no longer be determined solely by each branch's own uncertainty. Instead, they should be assigned by an independent arbitration module that has access to the global multi-view context. Motivated by this, we propose \textbf{\underline{T}rusted \underline{M}ulti-view learning with \underline{U}nified \underline{R}outing (TMUR)}, a framework that explicitly decouples view-specific evidence extraction from fusion arbitration. Specifically, TMUR equips each view, together with one global view, with its own expert network, and employs a unified router that observes the global multi-view context to generate sample-wise weights over these experts. In this way, fusion authority is no longer inferred only from branch-local uncertainty, but is instead determined through global, sample-aware routing. This design preserves the trusted nature of evidential prediction while mitigating the fusion bias induced by raw cross-view evidence-scale mismatch.

We also provide a theoretical analysis explaining why independent evidence supervision cannot determine a universal evidence scale across views, and when view reliability depends on the sample, global dynamic routing is better grounded than branch-local self-weighting. In summary, this work makes the following contributions:
\begin{itemize}
    \item We identify \textbf{cross-view evidence-scale incomparability} as an important yet underexplored bottleneck in trusted multi-view classification, and trace it to both data-level and model-level heterogeneity.
    \item We propose \textbf{TMUR}, a unified-routing framework that decouples evidential prediction from fusion arbitration and reduces scale-induced bias in cross-view fusion.
    \item We provide theoretical analysis clarifying why independent evidential supervision cannot guarantee cross-view evidence comparability, and why dynamic global routing is preferable when view reliability is sample-dependent.
    \item We validate TMUR on extensive multi-view benchmarks, reducing average ECE by 8.27 points and achieving consistently strong classification performance.   
\end{itemize}

\section{Related Work}
\subsection{Multi-View Learning}
Multi-view learning aims to exploit complementary yet heterogeneous views for stronger prediction. 

Recent work improves multi-view reasoning through stronger cross-view alignment~\cite{lin2025enhance,li2023crossview,yuan2025prototype}, trust-aware pairwise modeling~\cite{qin2026multi}, uncertainty-refined representation learning~\cite{hu2025self}, and agent-based cluster decision making~\cite{xue2026grok}.
In the classification setting, trusted multi-view methods such as TMC \cite{han2021trusted} and its dynamic evidential extension \cite{han2023dynamic} formulate the problem through view-wise evidence extraction and decision-level fusion. Subsequent studies further enhance this line by modeling inter-view conflict \cite{xu2024reliable}, leveraging trust-aware learning under conflicting views \cite{lu2025navigating}, introducing fuzzy evidence for reliable classification \cite{duan2025fuml}, addressing ambiguity and debiasing in open-set settings \cite{fang2025enhancing}, and incorporating expert knowledge constraints \cite{liang2025expert}. These efforts clearly demonstrate the importance of reliability-aware multi-view learning. Our work follows this trusted multi-view learning line, but focuses on a different bottleneck: whether the evidence produced by different branches is numerically comparable enough to support direct fusion.

\subsection{Evidential Fusion and Subjective Logic}
Evidential Deep Learning maps network outputs to Dirichlet parameters and provides a natural interface for predictive uncertainty \cite{sensoy2018evidential}, with later extensions modeling richer forms of evidential uncertainty \cite{li2024hyper}. In multi-view settings, these outputs are often combined through subjective logic or Dempster--Shafer style fusion \cite{shafer1976mathematical,josang2018subjective}. A related line estimates branch quality from uncertainty, total evidence, calibration, or belief discrepancy, and then reweights the resulting opinions accordingly \cite{han2023dynamic,cao2024predictive,fang2025enhancing,guo2017calibration,shi2026not}. However, these signals are still generated independently by each branch, so their magnitudes can remain entangled with branch-specific data and model biases rather than true sample-level reliability. In other words, reweighting branch-local opinions by branch-local quality estimates does not by itself make cross-view evidence numerically comparable. Classical evidence theory has long warned that fusion behavior becomes unreliable when the compatibility assumptions behind the combination rule are not justified \cite{zadeh1986simple}. Our method is motivated by this gap: instead of continuing to infer fusion authority directly from branch-local evidence quality, we delegate arbitration to a separate global router.

\subsection{Dynamic Routing and Expert-Based Arbitration}
Sample-dependent fusion has a long history in early expert-routing models \cite{jacobs1991adaptive} and modern routing architectures \cite{shazeer2017outrageously}. In multimodal learning, dynamic fusion has been used to adapt fusion weights to changing modality quality or sample context \cite{zhang2023provable,wu2025mvcge}, and predictive fusion methods also exploit sample-wise uncertainty signals to guide fusion decisions \cite{cao2024predictive}. At the same time, recent theory shows that dynamic fusion is not automatically robust: poorly designed weighting rules can even exacerbate modality greediness \cite{ding2025theoretical}. This makes the arbitration mechanism itself a first-class design issue. Recent routing-based expert models further show that expert specialization is useful across diverse multi-view or multimodal applications, including multimodal interaction modeling \cite{I2MoE}, molecular property prediction \cite{shirasuna2024moe_molecular}, time-series forecasting \cite{Huang2026M2FMoE}, and multi-view clustering \cite{zhang2025mixture}. Our use of routing is different. We introduce a unified router not as a generic capacity-expansion module, but as a dedicated arbitration mechanism that reduces the dependence of fusion authority on direct comparison of raw view-wise evidence.

\begin{figure*}[!ht]
\centering
\includegraphics[width=0.95\textwidth]{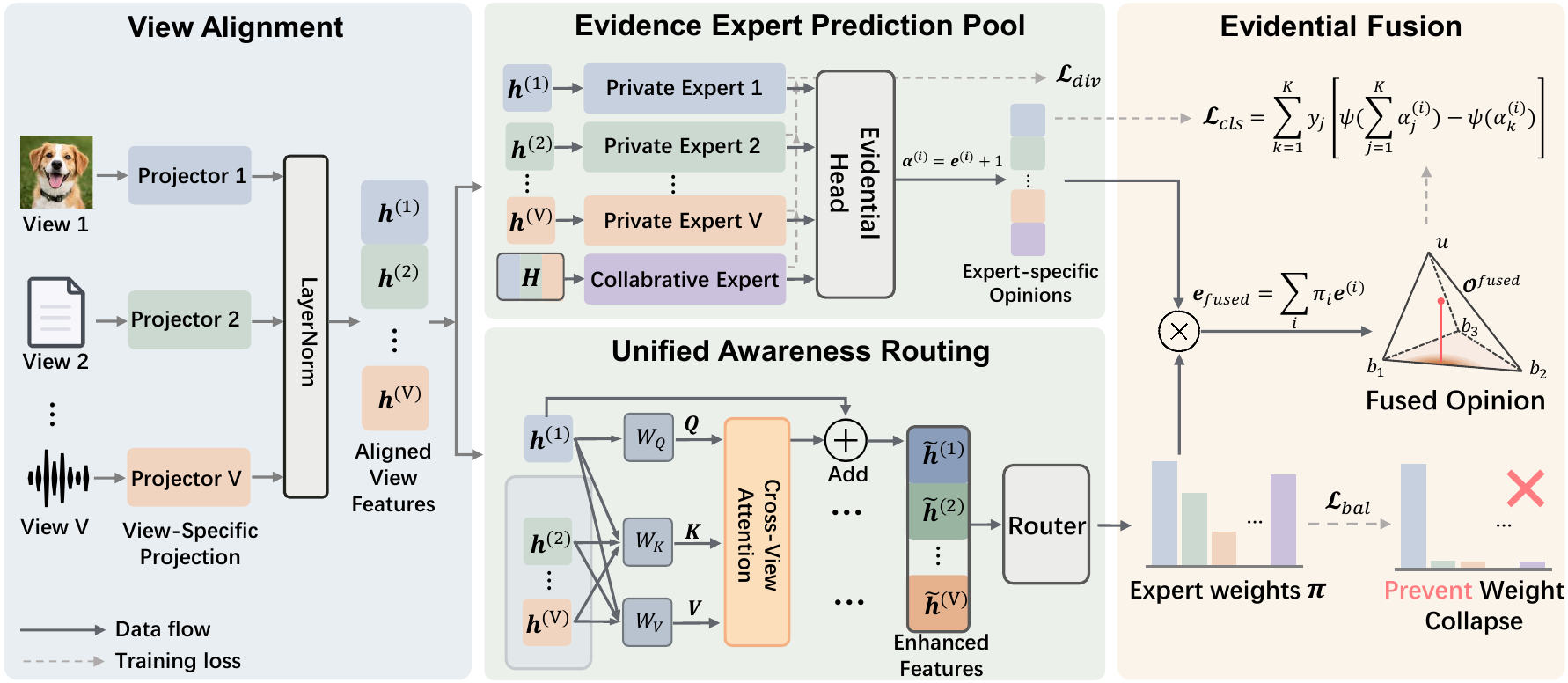}
\caption{Overview of TMUR. Aligned per-view features are sent to view-private experts for view-specific evidence extraction, while their concatenation is processed by a collaborative expert for joint multi-view evidence. In parallel, a unified router performs router-side cross-view interaction and predicts sample-wise expert weights from the global multi-view context. The final prediction is obtained by weighted evidence aggregation, so fusion authority is assigned by the router rather than by directly comparing branch-local evidence or uncertainty.}
\Description{Framework of TMUR. Aligned view features are fed to view-private experts and a collaborative expert for evidence extraction, while a separate unified router observes the multi-view context, computes expert weights, and controls the final weighted evidence fusion.}
\label{fig:framework}
\end{figure*}

\section{Method}
We propose a unified-routing framework for trusted multi-view classification. 
The core principle is to decouple \emph{view-wise evidence extraction} from \emph{multi-view fusion process}. 
Instead of letting each view-specific branch explicitly determine its own fusion authority through its local evidence magnitude or uncertainty, we assign this role to a specialized router that observes the joint multi-view context. 
In this way, the model retains the advantages of trusted multi-view learning, dynamic fusion and uncertainty awareness, while avoiding direct reliance on the comparability of raw cross-view evidence strengths.

\subsection{Problem Definition}
We consider a $V$-view classification problem with $K$ classes. 
The training set is denoted by:
\begin{equation}
\mathcal{D}=\{(\bm{x}_n^{(1)},\bm{x}_n^{(2)},\ldots,\bm{x}_n^{(V)},y_n)\}_{n=1}^{N},
\end{equation}
where $\bm{x}_n^{(v)} \in \mathbb{R}^{d_v}$ is the feature of the $v$-th view of the $n$-th sample, and $y_n \in \{1,\ldots,K\}$ is its class label. 
For simplicity, we omit the sample index $n$ below when no confusion arises.

In trusted multi-view learning~\cite{dong2025trusted}, each branch is expected to produce both class evidence and an uncertainty estimate. Following evidential learning, a predictor outputs a non-negative evidence vector $\bm e \in \mathbb{R}_{\ge 0}^{K}$, from which the Dirichlet parameters are defined as $\bm \alpha = \bm e + \bm 1.$
Let Dirichlet strength $S=\sum_{j=1}^{K}\alpha_j$. The corresponding predictive probability and subjective-logic uncertainty are
\begin{equation}
p_k = \frac{\alpha_k}{S},
\qquad
u = \frac{K}{S}.
\end{equation}
Equivalently, the associated opinion can be written as $\mathcal{O}=(\bm b,u)$ with belief masses
\begin{equation}
b_k = \frac{e_k}{S}, \qquad k=1,\ldots,K.
\end{equation}
This formulation makes the uncertainty directly depend on the total evidence magnitude. As a result, if different views produce evidence on different numerical scales, then their induced uncertainties need not be directly comparable across views, even when their semantic support patterns are similar.


\subsection{Unified Routing and Evidential Fusion}


To address this issue, \textbf{TMUR} explicitly separates two roles that are often entangled in existing trusted multi-view fusion methods: view-wise evidence extraction and cross-view fusion arbitration. The first role is to let each view-specific branch produce trusted evidence while preserving its own semantics. The second role is to decide, under the joint multi-view context of the current sample, how much each expert should contribute to the final decision. Based on this separation, TMUR maintains an expert pool composed of $V$ private experts and one collaborative expert, and introduces a unified router that assigns sample-wise fusion weights from the joint multi-view context. The overall architecture is shown in Fig.~\ref{fig:framework}.

\subsubsection{\textbf{Router and Experts Settings}}
Specifically, we first map all views into a dimension-aligned space:
\begin{equation}
\bm h^{(v)} = \mathrm{LN}(P_v(\bm x^{(v)})), \qquad v=1,\ldots,V,
\end{equation}
where $P_v(\cdot)$ is a view-specific projection layer and $\mathrm{LN}(\cdot)$ denotes LayerNorm. This step improves cross-view feature compatibility while retaining view-specific information. We then concatenate the aligned view features as:
\begin{equation}
\bm H = [\bm h^{(1)};\bm h^{(2)};\cdots;\bm h^{(V)}].
\end{equation}

The expert pool contains $V$ private experts and one collaborative expert. Each expert is implemented as a MLP. The $v$-th private expert receives only the aligned feature of the $v$-th view,
\begin{equation}
\bm z^{(v)} = E_v(\bm h^{(v)}), \qquad v=1,\ldots,V,
\end{equation}
while the collaborative expert receives the concatenated multi-view representation,
\begin{equation}
\bm z^{(V+1)} = E_{V+1}(\bm H).
\end{equation}
This design allows the private experts to express view-specific evidence, while the collaborative expert may capture complementary cues that emerge only after combining multiple views.

To provide the router with richer cross-view context, we further construct interaction-enhanced features through cross-view attention:
\begin{equation}
\tilde{\bm h}^{(v)} = \mathrm{Attn}\!\left(\bm h^{(v)}, \{\bm h^{(j)}\}_{j=1}^{V}, \{\bm h^{(j)}\}_{j=1}^{V}\right),
\end{equation}
where the query is the aligned feature of the current view and the keys and values are the aligned features from all views. We then concatenate the enhanced features as:
\begin{equation}
\bm g = [\tilde{\bm h}^{(1)};\tilde{\bm h}^{(2)};\cdots;\tilde{\bm h}^{(V)}].
\end{equation}
These enhanced features are used to form the routing context, whereas the private experts still operate on their own aligned view features. In this way, cross-view interaction affects how experts are weighted, rather than replacing the view-specific evidence produced by the private experts themselves.

Each expert outputs a non-negative evidence vector through an evidential classification head:
\begin{equation}
\bm e^{(i)} = \phi(\bm z^{(i)}), \qquad i \in \{1,\ldots,V+1\},
\end{equation}
where $\phi(\cdot)$ is a non-negative activation function. We use a \texttt{softplus}-based implementation in practice. The corresponding Dirichlet parameters are $\bm \alpha^{(i)} = \bm e^{(i)} + \bm 1$. Let $S^{(i)} = \sum_{k=1}^{K}\alpha_k^{(i)}$. Then the opinion associated with expert $i$ is denoted by
\begin{equation}
\mathcal{O}^{(i)} = (\bm b^{(i)}, u^{(i)}).
\end{equation}

\subsubsection{\textbf{Router-Driven Evidential Fusion}}
If raw evidential magnitude is not a reliable cross-view arbitration signal, then fusion weights should not be inferred directly from branch-local evidence or uncertainty alone. Instead, they should be assigned by a dedicated routing mechanism that can observe all views jointly. Based on the enhanced multi-view context $\bm g$, we compute the sample-wise expert weights as:
\begin{equation}
\bm \pi = \mathrm{softmax}\!\left(\frac{R(\bm g)}{\tau}\right)
\in \mathbb{R}^{V+1},
\end{equation}
where $R(\cdot)$ denotes the routing network and $\tau$ is the routing temperature.

The final fused evidence is obtained by weighted aggregation:
\begin{equation}
\bm e^{\mathrm{fused}}
=
\sum_{i=1}^{V+1} \pi_i \bm e^{(i)}.
\end{equation}
The corresponding fused opinion is denoted by
\begin{equation}
\mathcal{O}^{\mathrm{fused}} = (\bm b^{\mathrm{fused}}, u^{\mathrm{fused}}),
\end{equation}
where
\begin{equation}
b_k^{\mathrm{fused}} = \frac{e_k^{\mathrm{fused}}}{S^{\mathrm{fused}}},
\qquad
u^{\mathrm{fused}} = \frac{K}{S^{\mathrm{fused}}}.
\end{equation}
The predictive class probabilities can still be obtained from the fused Dirichlet mean:
\begin{equation}
p_k^{\mathrm{fused}} = \frac{\alpha_k^{\mathrm{fused}}}{S^{\mathrm{fused}}}.
\end{equation}

This fusion mechanism preserves the evidential output form while reassigning the arbitration role. Existing trusted multi-view fusion methods often let each branch both produce evidence and implicitly influence its own fusion weight through branch-local uncertainty. In contrast, TMUR keeps evidence generation inside the experts and lets a unified router assign fusion authority from the joint multi-view context. Consequently, the final weights are guided primarily by sample-level cross-view context, which reduces their dependence on branch-specific evidential scale.

\subsection{Training Objective}
Our training objective is composed of several complementary terms that respectively supervise the fused prediction, maintain expert discriminability, and regularize the routing behavior.

\paragraph{Fused evidential classification loss.}
We take the fused output as the primary supervision target and optimize it with the evidential digamma loss~\cite{chen2025revisiting}:
\begin{equation}
\mathcal{L}_{\mathrm{fused}}
=
\sum_{k=1}^{K}
y_k
\left[
\psi\!\left(\sum_{j=1}^{K}\alpha_j^{\mathrm{fused}}\right)
-
\psi\!\left(\alpha_k^{\mathrm{fused}}\right)
\right],
\end{equation}
where $\psi(\cdot)$ is the digamma function and $\bm y$ is the one-hot label vector.

\paragraph{Auxiliary expert supervision.}
Although the final prediction is produced after fusion, each expert should still learn to generate meaningful class evidence. We therefore apply the same evidential supervision to every expert:
\begin{equation}
\mathcal{L}_{\mathrm{view}}
=
\frac{1}{V+1}
\sum_{i=1}^{V+1}
\sum_{k=1}^{K}
y_k
\left[
\psi\!\left(\sum_{j=1}^{K}\alpha_j^{(i)}\right)
-
\psi\!\left(\alpha_k^{(i)}\right)
\right].
\end{equation}
This auxiliary supervision stabilizes expert learning and prevents the router from compensating for poorly trained experts.

\paragraph{Load-balancing regularization.}
Since multi-view data can contain persistently strong and weak views, we do not force the router to use all experts equally. Instead, we only penalize excessive concentration of expert usage within a mini-batch, so that routing remains adaptive while avoiding collapse. Let
\begin{equation}
\bar{\pi}_i = \frac{1}{B}\sum_{b=1}^{B}\pi_{b,i}
\end{equation}
be the average routing weight of expert $i$ over a mini-batch of size $B$, and define the routing concentration as
\begin{equation}
\mathrm{Conc}(\bar{\bm \pi})=\sum_i \bar{\pi}_i^2.
\end{equation}
The load-balancing loss is
\begin{equation}
\mathcal{L}_{\mathrm{bal}}
=
\max\!\left(
\mathrm{Conc}(\bar{\bm \pi}) - \frac{\rho}{V+1},\ 0
\right),
\end{equation}
where $\rho>1$ controls the tolerated concentration.

\paragraph{\textbf{Expert diversity regularization.}}
Because the private experts are optimized jointly, they may still drift toward redundant hidden representations. To encourage specialization, we regularize their hidden features to be decorrelated. 
Let $\bm q^{(v)}$ denote the hidden representation produced by the $v$-th private expert before its evidential classification head, and let $\hat{\bm q}^{(v)}=\bm q^{(v)}/\|\bm q^{(v)}\|_2$. We define
\begin{equation}
\mathcal{L}_{\mathrm{div}}
=
\frac{2}{V(V-1)}
\sum_{1 \le i < j \le V}
\left(
\hat{\bm q}^{(i)\top}\hat{\bm q}^{(j)}
\right)^2.
\end{equation}
The squared cosine form penalizes both positive and negative alignment, thereby encouraging different private experts to capture different view-conditioned directions.

The overall training objective is
\begin{equation}
\mathcal{L}
=
\mathcal{L}_{\mathrm{fused}}
+
\lambda \mathcal{L}_{\mathrm{view}}
+
\beta \mathcal{L}_{\mathrm{bal}}
+
\gamma \mathcal{L}_{\mathrm{div}},
\end{equation}
where $\lambda$, $\beta$, and $\gamma$ control the strengths of auxiliary expert supervision, load balancing, and diversity regularization, respectively.

\section{Theoretical Analysis}

This section formalizes two points that directly support our design. First, branch-local uncertainty is scale-sensitive: changing evidence magnitude changes subjective-logic uncertainty even when the underlying class-support pattern is unchanged. Second, when optimal fusion authority depends on joint cross-view context, branch-local self-weighting is fundamentally limited, whereas a unified router can exploit the required joint information.

\subsection{Why branch-local uncertainty is scale-sensitive}

For an expert outputting non-negative evidence $\bm e \in \mathbb{R}_{\ge 0}^{K}$, define
\begin{equation}
\bm \alpha = \bm e + \bm 1, \qquad
S = \sum_{k=1}^{K}\alpha_k, \qquad
u = \frac{K}{S}.
\end{equation}
Thus, uncertainty is inversely proportional to the total evidence scale.

Consider a fixed non-negative support pattern $\bm r \in \mathbb{R}_{\ge 0}^{K}$ and define a one-parameter evidence family
\begin{equation}
\bm e(t)=t \bm r, \qquad t>0.
\end{equation}
Then
\begin{equation}
\bm \alpha(t)=\bm 1+t \bm r,
\qquad
u(t)=\frac{K}{K+tR},
\end{equation}
where $R=\sum_{k=1}^{K}r_k$.

\begin{theorem}[Scale changes uncertainty without changing support direction]
\label{thm:scale_bias}
For the above family,
\begin{equation}
\frac{d\,u(t)}{dt}
=
-\frac{KR}{(K+tR)^2}
<0.
\end{equation}
Moreover, if $r_y > R/K$ for the ground-truth class $y$, then the predictive probability of the true class increases with $t$ as well. Therefore, along the same class-support pattern, increasing the evidence scale decreases the induced uncertainty even though the underlying support direction is unchanged.
\end{theorem}

\noindent\textit{Proof.}
The expression for $u(t)$ follows directly from the Dirichlet definition, and differentiating with respect to $t$ gives the stated derivative. If $r_y>R/K$, the predictive probability of the true class is also increasing in $t$. \qed

Theorem~\ref{thm:scale_bias} explains why branch-local uncertainty is vulnerable to scale bias. Two branches may encode essentially the same relative class preference, yet the one with a larger evidence scale will appear less uncertain and hence more trusted.

\begin{table*}[t]
\centering
\caption{Five-seed benchmark accuracy on the first seven reported datasets. Each entry reports mean $\pm$ standard deviation over five seeds. Method names include citation tags, while the `References' column lists venue and year.}
\label{tab:main-results-front}
\setlength{\tabcolsep}{5.8pt}
\renewcommand{\arraystretch}{0.95}
\begin{tabular}{l|l|ccccccc}
\toprule
\textbf{Method} & \textbf{Ref.} & \textbf{HandWritten} & \textbf{Scene} & \textbf{LandUse} & \textbf{NUS} & \textbf{Caltech-6V} & \textbf{PIE} & \textbf{WebKB}\\
\midrule
TMC~\cite{han2021trusted} & ICLR'21 & 98.00$\pm$0.88 & 75.76$\pm$0.95 & 60.95$\pm$2.25 & 44.70$\pm$0.36 & 93.72$\pm$0.30 & 91.32$\pm$2.43 & 82.44$\pm$2.84 \\
TMDLO~\cite{liu2022trusted} & AAAI'22 & 97.20$\pm$1.03 & 66.87$\pm$2.97 & 52.67$\pm$3.36 & 43.09$\pm$0.42 & 86.15$\pm$2.42 & 73.09$\pm$1.89 & 83.41$\pm$2.84 \\
ETMC~\cite{han2023dynamic} & PAMI'23 & 98.30$\pm$0.48 & 78.10$\pm$1.14 & 65.48$\pm$1.27 & \underline{48.53$\pm$0.34} & 94.06$\pm$0.47 & 94.85$\pm$1.14 & 82.93$\pm$2.67 \\
RCML~\cite{xu2024reliable} & AAAI'24 & 98.00$\pm$0.79 & 76.03$\pm$1.35 & 65.86$\pm$1.72 & 44.43$\pm$0.28 & 94.60$\pm$0.38 & 95.29$\pm$1.10 & 82.93$\pm$2.67 \\
ETF~\cite{lu2025navigating} & ICML'25 & 96.65$\pm$1.12 & 68.21$\pm$0.79 & 19.52$\pm$8.86 & 44.12$\pm$0.69 & 93.97$\pm$0.48 & 95.59$\pm$1.40 & 80.49$\pm$3.09 \\
FUML~\cite{duan2025fuml} & ICML'25 & 98.90$\pm$0.54 & 80.65$\pm$1.03 & 76.62$\pm$1.24 & 48.07$\pm$0.34 & 95.61$\pm$0.37 & 96.91$\pm$1.35 & 72.20$\pm$18.09 \\
TMCEK~\cite{liang2025expert} & ICML'25 & 98.30$\pm$0.84 & 77.37$\pm$1.20 & 70.10$\pm$1.18 & 43.20$\pm$0.32 & 95.40$\pm$0.51 & \underline{97.79$\pm$0.93} & \underline{85.37$\pm$2.67} \\
TUNED~\cite{huang2025trusted} & AAAI'25 & 98.90$\pm$0.34 & 78.15$\pm$0.69 & 72.95$\pm$1.80 & 43.35$\pm$0.31 & 95.15$\pm$0.75 & 89.56$\pm$2.56 & 80.00$\pm$2.39 \\
TEF~\cite{liang2025evolutionary} & ICLR'25 & 98.80$\pm$0.51 & 78.00$\pm$0.48 & \textbf{80.48$\pm$1.41} & 47.14$\pm$0.52 & 94.85$\pm$0.28 & 97.65$\pm$1.27 & 84.88$\pm$2.39 \\
SAEML~\cite{xu2025beyond} & MM'25 & 98.80$\pm$0.53 & 79.67$\pm$1.42 & 74.05$\pm$1.23 & 42.94$\pm$0.25 & 93.81$\pm$0.51 & 95.88$\pm$1.78 & 84.88$\pm$2.39 \\
RTMC~\cite{zhou2025refining} & WWW'25 & 84.55$\pm$1.61 & 68.92$\pm$1.31 & 54.38$\pm$0.36 & 35.12$\pm$0.69 & 94.60$\pm$0.33 & 92.79$\pm$1.99 & 78.54$\pm$2.39 \\
RCMCL~\cite{hu2026robust} & PAMI'26 & 97.15$\pm$1.06 & 70.81$\pm$1.37 & 45.24$\pm$1.15 & 37.75$\pm$0.28 & 94.14$\pm$0.70 & 96.18$\pm$1.18 & 81.95$\pm$2.49 \\
\midrule
NLC~\cite{xu2025noisy} & AAAI'25 & 98.30$\pm$0.73 & 80.38$\pm$1.56 & 70.29$\pm$1.15 & 47.85$\pm$0.49 & 92.34$\pm$0.81 & 89.71$\pm$2.42 & 71.71$\pm$9.58 \\
MAMC~\cite{lin2025enhance} & ICLR'25 & 98.75$\pm$0.47 & \underline{81.50$\pm$0.70} & \underline{79.33$\pm$0.59} & 48.14$\pm$0.34 & \underline{96.03$\pm$0.30} & 97.06$\pm$1.04 & 84.39$\pm$1.19 \\
BCM~\cite{lan2025BCM} & MM'25 & \underline{99.00$\pm$0.42} & 80.71$\pm$0.60 & 78.81$\pm$1.86 & 47.25$\pm$0.46 & 94.35$\pm$0.84 & 94.56$\pm$1.95 & \underline{85.37$\pm$3.09} \\
\midrule
\rowcolor[RGB]{234, 238, 234}
TMUR & Ours & \textbf{99.10$\pm$0.46} & \textbf{82.88$\pm$1.36} & \textbf{80.48$\pm$0.78} & \textbf{50.07$\pm$0.26} & \textbf{96.69$\pm$0.36} & \textbf{97.94$\pm$0.86} & \textbf{88.29$\pm$0.98} \\
\bottomrule
\end{tabular}
\end{table*}

\subsection{Why unified routing is preferable to branch-local self-weighting}

Let $\mathcal{L}_x(\bm w)$ denote the sample-wise fusion loss as a function of the expert-weight vector
\begin{equation}
\bm w \in \Delta^{V},
\end{equation}
where $\Delta^{V}$ is the probability simplex over the $V+1$ experts. Define the oracle routing rule
\begin{equation}
\bm w^\star(x)=\arg\min_{\bm w \in \Delta^{V}} \mathcal{L}_x(\bm w).
\end{equation}
Suppose branch-local self-weighting rules can only access local statistics $\bm s(x)$, such as branch-wise evidence or uncertainty, and therefore take the form
\begin{equation}
\bm w_{\mathrm{local}}(x)=\psi(\bm s(x)).
\end{equation}
By contrast, a unified router can access a joint representation $\bm g(x)$ and output
\begin{equation}
\bm w_{\mathrm{global}}(x)=\phi(\bm g(x)).
\end{equation}

\begin{theorem}[Branch-local self-weighting has an irreducible information gap]
\label{thm:router_advantage}
Assume that for each sample $x$, the fusion loss $\mathcal{L}_x(\bm w)$ is $\mu$-strongly convex on $\Delta^{V}$ for some $\mu>0$. Let
\begin{equation}
\mathcal{W}_{\mathrm{local}}
:=
\{\bm w(x)=\psi(\bm s(x))\}
\end{equation}
denote the class of branch-local self-weighting rules. Then
\begin{equation}
\inf_{\bm w\in\mathcal{W}_{\mathrm{local}}}
\mathbb{E}\!\left[\mathcal{L}_x(\bm w(x))\right]
-
\mathbb{E}\!\left[\mathcal{L}_x(\bm w^\star(x))\right]
\ge
\frac{\mu}{2}
\mathbb{E}\!\left[
\mathrm{Var}\!\left(\bm w^\star(x)\mid \bm s(x)\right)
\right].
\end{equation}
In particular, if the optimal authority assignment depends on cross-view relations that are not recoverable from branch-local statistics alone, then every branch-local self-weighting rule incurs a strictly positive irreducible excess risk.
\end{theorem}

\noindent\textit{Proof sketch.}
Strong convexity lower-bounds excess loss by squared distance to the oracle weight $\bm w^\star(x)$. Among all rules measurable with respect to $\bm s(x)$, the best predictor of $\bm w^\star(x)$ under squared loss is the conditional expectation, and the corresponding minimum residual is the conditional variance. Therefore, if $\bm w^\star(x)$ is not measurable from branch-local statistics alone, branch-local self-weighting necessarily leaves a positive irreducible gap. \qed

\section{Experiments}
\subsection{Experimental Setup}
\paragraph{\textbf{Dataset.}}
We evaluate TMUR on fourteen public multi-view classification datasets\footnote{https://github.com/JethroJames/Awesome-Multi-View-Learning-Datasets}: HandWritten, Scene, LandUse, NUS, Caltech-6V, PIE, WebKB, Caltech, UCI, CUB, Animal, MSRCV1, BBC, and Leaves. Detailed dataset statistics, view descriptions are deferred to the appendix.

\begin{table*}[t]
\centering
\caption{Calibration results on trusted multi-view classification methods (ECE\%, lower is better).}
\label{tab:calibration_ece}
\setlength{\tabcolsep}{3pt}
\renewcommand{\arraystretch}{1}
\begin{tabular}{l|l|cccccccc}
\toprule
\textbf{Method} & \textbf{Ref.} 
& \textbf{HandWritten} & \textbf{Scene} & \textbf{LandUse} & \textbf{PIE} & \textbf{WebKB} & \textbf{Animal} & \textbf{Caltech-6V} & \textbf{Leaves} \\
\midrule
TMC~\cite{han2021trusted} & ICLR'21
& 13.34$\pm$2.13 & 46.07$\pm$1.74 & 40.29$\pm$2.33 & 74.50$\pm$0.53 & \textbf{8.91$\pm$2.95} & 43.31$\pm$7.20 & \underline{6.43$\pm$4.11} & 87.56$\pm$2.59 \\
TMDLO~\cite{liu2022trusted} & AAAI'22
& 27.16$\pm$4.29 & 23.21$\pm$2.39 & 26.71$\pm$2.62 & \underline{58.25$\pm$2.28} & 32.36$\pm$11.21 & 25.60$\pm$1.09 & 18.16$\pm$2.05 & \underline{50.23$\pm$1.91} \\
ETMC~\cite{han2023dynamic} & PAMI'23
& 8.69$\pm$2.60 & 40.03$\pm$3.55 & 41.78$\pm$1.53 & 84.49$\pm$1.53 & 21.56$\pm$6.58 & \underline{23.93$\pm$10.13} & 7.76$\pm$3.54 & 87.74$\pm$3.86 \\
RCML~\cite{xu2024reliable} & AAAI'24
& 64.90$\pm$0.81 & 58.64$\pm$1.61 & 51.92$\pm$1.91 & 87.05$\pm$1.80 & 38.95$\pm$4.01 & 51.11$\pm$2.23 & 47.07$\pm$1.47 & 90.43$\pm$1.15 \\
TEF~\cite{liang2025evolutionary} & ICLR'25
& \underline{8.13$\pm$5.73} & \underline{10.83$\pm$2.06} & 24.58$\pm$2.18 & 70.47$\pm$4.63 & 47.51$\pm$8.00 & 24.91$\pm$0.52 & 9.21$\pm$2.02 & 72.84$\pm$1.71 \\
TMCEK~\cite{liang2025expert} & ICML'25
& 62.25$\pm$4.67 & 61.33$\pm$1.97 & 58.32$\pm$0.99 & 92.16$\pm$0.73 & 28.62$\pm$5.77 & 54.89$\pm$2.00 & 54.25$\pm$4.19 & 93.83$\pm$0.67 \\
TUNED~\cite{huang2025trusted} & AAAI'25
& 48.83$\pm$0.67 & 55.69$\pm$1.17 & 52.62$\pm$1.14 & 81.17$\pm$1.55 & 31.17$\pm$2.90 & 64.18$\pm$1.24 & 52.09$\pm$1.42 & 84.56$\pm$1.52 \\
SAEML~\cite{xu2025beyond} & MM'25
& 9.83$\pm$1.71 & 12.45$\pm$0.99 & \underline{16.62$\pm$1.22} & 60.35$\pm$4.20 & \underline{10.11$\pm$3.12} & 30.68$\pm$0.71 & \textbf{5.01$\pm$1.10} & 66.96$\pm$1.17 \\
RTMC~\cite{zhou2025refining} & WWW'25
& 72.39$\pm$1.70 & 59.47$\pm$1.32 & 47.88$\pm$0.36 & 90.69$\pm$1.98 & 47.41$\pm$3.69 & 81.50$\pm$0.97 & 87.16$\pm$0.34 & 84.50$\pm$2.54 \\
RCMCL~\cite{hu2026robust} & PAMI'26
& 51.09$\pm$4.38 & 52.62$\pm$1.38 & 35.62$\pm$1.16 & 86.92$\pm$1.12 & 42.86$\pm$3.71 & 60.70$\pm$0.71 & 42.15$\pm$1.99 & 93.61$\pm$0.72 \\
\midrule
\rowcolor[RGB]{234, 238, 234}
TMUR & Ours
& \textbf{5.36$\pm$1.66} & \textbf{6.87$\pm$1.69} & \textbf{7.35$\pm$0.94} & \textbf{52.50$\pm$3.93} & 12.01$\pm$4.44 & \textbf{9.79$\pm$5.73} & 7.55$\pm$5.23 & \textbf{44.43$\pm$8.60} \\
\bottomrule
\end{tabular}
\end{table*}

\paragraph{\textbf{Compared methods.}} 
We compare TMUR with two groups of baselines. The trusted baselines include TMC~\cite{han2021trusted}, TMDLO~\cite{liu2022trusted}, and ETMC~\cite{han2023dynamic}, which formulate multi-view predictions as evidence or opinions and perform uncertainty-aware fusion.
They also include more recent trusted methods for conflict handling and reliability modeling, including RCML~\cite{xu2024reliable}, ETF~\cite{lu2025navigating}, FUML~\cite{duan2025fuml}, and RCMCL~\cite{hu2026robust}, which focus on conflict-aware aggregation, trust discounting, fuzzy uncertainty modeling, or robust collaborative learning under conflictive views.
In addition, we compare with recent refinements of trusted fusion, including TMCEK~\cite{liang2025expert}, TUNED~\cite{huang2025trusted}, TEF~\cite{liang2025evolutionary}, SAEML~\cite{xu2025beyond}, and RTMC~\cite{zhou2025refining}, which introduce expert knowledge constraints, feature-neighborhood dynamics, evolutionary fusion architecture search, strength-adaptive evidence weighting, and a refined treatment of confusion and ignorance, respectively.
The non-trusted baselines include NLC~\cite{xu2025noisy}, which addresses noisy-label multi-view classification via label calibration, MAMC~\cite{lin2025enhance}, which combines multi-scale alignment with expanded decision boundaries, and BCM~\cite{lan2025BCM}, which performs multi-view classification in Hamming space with a bit-level calibration mechanism.

\paragraph{\textbf{Implementation details.}}
All experiments are conducted on NVIDIA RTX 4090 GPUs using PyTorch 2.2.2. Following our benchmark protocol, every method is evaluated with five fixed random seeds, and the reported accuracy is the mean $\pm$ standard deviation over the five runs. We use Adam with an initial learning rate of $10^{-3}$, cosine learning-rate decay, and a fixed batch size of 128. The auxiliary expert-supervision weight $\lambda$ is fixed to 0.3 on all datasets, while the load-balancing and diversity weights $\beta$ and $\gamma$ are selected per dataset; their final values are reported in the appendix.


\subsection{Main Results} 
Table~\ref{tab:main-results-front} reports the five-seed accuracy results on seven representative datasets, with each entry shown as mean $\pm$ standard deviation. To keep the comparison readable under the enlarged baseline set, the methods are separated into trusted baselines, non-trusted baselines, and our model, with horizontal rules marking the three groups. The remaining are reported in appendix.

TMUR achieves the best or tied-best accuracy on all of the seven datasets in Table~\ref{tab:main-results-front}, including LandUse, NUS, PIE, Caltech-6V, and WebKB. In particular, the gains on NUS and WebKB are clear, suggesting that unified routing is especially helpful when the reliability of different views varies substantially across samples. The appendix results on the remaining seven datasets show the same overall pattern: TMUR remains highly competitive and achieves the best or tied-best performance on six datasets, indicating that the proposed routing mechanism improves accuracy without sacrificing generality across diverse multi-view benchmarks.

\begin{figure}[!ht]
\centering
\includegraphics[width=\columnwidth]{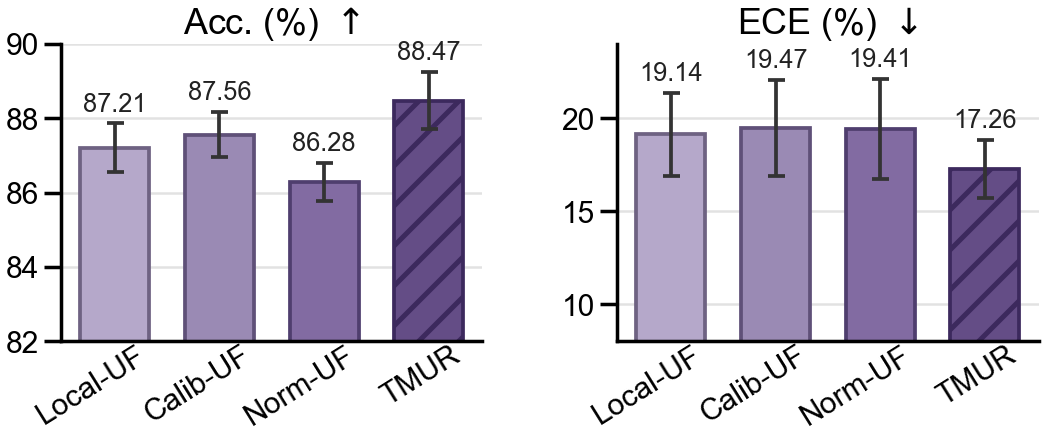}
\vskip -0.1in
\caption{Comparison with branch-local uncertainty-fusion variants on six datasets. Results are averaged over five runs; error bars denote standard deviations.}
\Description{Two bar plots compare Local-UF, Calib-UF, Norm-UF, and TMUR in terms of average accuracy and uncertainty ECE. TMUR obtains the highest accuracy and the lowest ECE.}
\label{fig:uncertainty_aware}
\vskip -0.1in
\end{figure}

\begin{table*}[!t]
\centering
\caption{Ablation study of different components in our method.}
\label{tab:ablation}
\setlength{\tabcolsep}{3.6pt}
\begin{tabular}{l|cccccccc}
\toprule
\textbf{Variant}
& \textbf{HandWritten} & \textbf{Scene} & \textbf{LandUse} & \textbf{PIE} & \textbf{WebKB} & \textbf{Animal} & \textbf{Caltech-6V} & \textbf{Leaves} \\
\midrule
w/o $\mathcal{L}_{\mathrm{bal}}$
& 98.85$\pm$0.46 & \underline{82.63}$\pm$1.63 & \underline{80.05$\pm$0.87} & 97.50$\pm$0.88 & 87.80$\pm$1.54 & \underline{87.62$\pm$0.60} & 96.28$\pm$0.48 & 99.44$\pm$0.23 \\

w/o $\mathcal{L}_{\mathrm{div}}$
& \underline{99.00$\pm$0.45} & 82.23$\pm$1.47 & 79.48$\pm$0.73 & \underline{97.79$\pm$0.93} & \textbf{88.29$\pm$0.98} & 87.55$\pm$0.71 & 96.40$\pm$0.45 & \underline{99.62$\pm$0.23} \\
w/o Cross-Attention
& 98.95$\pm$0.51 & 82.43$\pm$1.70 & 78.95$\pm$0.83 & 97.65$\pm$0.72 & 87.80$\pm$1.54 & 87.58$\pm$0.59 & \underline{96.61$\pm$0.24} & \underline{99.62$\pm$0.13} \\
Full method
& \textbf{99.10$\pm$0.46} & \textbf{82.88$\pm$1.36} & \textbf{80.48$\pm$0.78} & \textbf{97.94$\pm$0.86} & \textbf{88.29$\pm$0.98} & \textbf{87.83$\pm$0.57} & \textbf{96.69$\pm$0.36} & \textbf{99.69$\pm$0.20} \\
\bottomrule
\end{tabular}
\end{table*}

\begin{table}[t]
\centering
\setlength{\tabcolsep}{3pt}
\renewcommand{\arraystretch}{0.78}
\caption{View Strength Variation Experiments. The view strength in the multi-view dataset is randomly rescaled.}
\label{tab:view_strength_variation}
\begin{tabular}{cc|cccc|c}
\toprule
\multirow{2}{*}{\textbf{Dataset}} & \multirow{2}{*}{\textbf{Setting}} & \multicolumn{5}{c}{\textbf{Methods}} \\
\cmidrule(lr){3-7}
& & \textbf{FUML} & \textbf{RCML} & \textbf{TMC} & \textbf{SAEML} & \textbf{Ours} \\
\midrule
\multirow{2}{*}{\textbf{PIE}}
& R1 & 92.89 & 93.38 & 88.73 & 93.87 & \textbf{96.08} \\
& R2 & 82.35 & 89.95 & 81.86 & 84.80 & \textbf{91.67} \\
\midrule
\multirow{2}{*}{\textbf{Scene}}
& R1 & 67.41 & 61.32 & 68.75 & 69.75  & \textbf{71.39}\\
& R2 & 69.12 & 63.43 & 71.16 & 73.39  & \textbf{74.99} \\
\midrule
\multirow{2}{*}{\textbf{LandUse}}
& R1  & 45.40 & 57.62 & 52.78 & 49.29  & \textbf{63.97} \\
& R2  & 42.14 & 54.68 & 56.90 & 52.78  & \textbf{65.16} \\
\midrule
\multirow{2}{*}{\textbf{Caltech-6V}}
& R1 & 94.70 & 94.56 & 93.93 & 92.89 & \textbf{95.33} \\
& R2 & 93.17 & \textbf{94.98} & 93.86 & 88.15 & 94.84 \\
\midrule
\multirow{2}{*}{\textbf{WebKB}}
& R1 & 64.23 & 81.30 & 82.11 & 81.30 & \textbf{86.18} \\
& R2 & 61.79 & 77.24 & 79.67 & 78.86 & \textbf{86.99} \\
\bottomrule
\end{tabular}%
\end{table}

Accuracy alone is insufficient for trustworthy multi-view learning; the fused prediction should also remain well calibrated. Table~\ref{tab:calibration_ece} therefore compares ECE on eight representative datasets against trusted baselines. TMUR achieves the lowest ECE on six of the eight datasets, namely HandWritten, Scene, LandUse, PIE, Animal, and Leaves, and remains competitive on the remaining two datasets. Averaged over all eight datasets, TMUR reduces ECE from 26.50\% achieved by the strongest baseline SAEML to 18.23\%, yielding an absolute improvement of more than 8 ECE points. These results are important because they show that our gains are not merely due to higher classification accuracy; instead, the proposed unified routing also produces more reliable confidence estimates under cross-view evidence mismatch. Taken together, Table~\ref{tab:main-results-front} and Table~\ref{tab:calibration_ece} show that TMUR improves both \emph{what} the model predicts and \emph{how reliably} it expresses confidence in that prediction.

\subsection{Mechanism and Robustness Evaluation}
\paragraph{\textbf{Comparison with calibrated/normalized local fusion.}}

To address whether branch-local self-assessed uncertainty becomes sufficient after calibration or normalization, we compare TMUR with three variants that use the same evidence outputs from the view-private branches and the collaborative branch, but replace the unified router with branch-local \textbf{u}ncertainty \textbf{f}usion. \textbf{Local-UF} weights each branch by raw evidential uncertainty; \textbf{Calib-UF} calibrates branch uncertainties on the validation set before fusion; \textbf{Norm-UF} normalizes branch evidence strength using validation-set statistics before fusion. We conducted experiments on six datasets using a 7:1:2 split. Fig.~\ref{fig:uncertainty_aware} shows that calibration or evidence normalization alone does not remove the need for unified arbitration. TMUR improves average accuracy by 0.91 points over the strongest local variant while reducing ECE. 
This suggests that the key benefit comes from learning sample-wise fusion weights from the joint multi-view context, rather than only rescaling branch-local uncertainty or evidence strength through post-hoc calibration.


\paragraph{\textbf{Router--reliability alignment.}}
We further examine whether the learned routing weights align with expert reliability on Caltech-6V, LandUse, Scene, and WebKB. Averaged over five seeds, the pooled sample--expert correlation between routing weight and expert correctness is positive ($\rho_s=0.179$), while the expert-level mean routing weight shows substantial Pearson and Spearman correlations with standalone expert accuracy ($\rho_e=0.776$ and $0.794$, respectively), indicating that the router tends to allocate larger weights to more reliable experts.

\paragraph{\textbf{View-strength variation.}}

Following Xu et al.~\cite{xu2025beyond}, we perturb the feature magnitude of each view at test time by randomly scaling its feature vector. For each random seed, R1 and R2 denote two independently sampled view-wise scaling configurations, with each scaling factor sampled from $[0.05,0.99]$. Table~\ref{tab:view_strength_variation} reports the results on five datasets. TMUR achieves the best performance in nine out of ten settings, showing stronger robustness when view strength is distorted. This supports our motivation that unified routing provides more stable fusion than branch-local confidence under cross-view scale variation.

\begin{figure}[!ht]
\centering
\includegraphics[width=0.95\columnwidth]{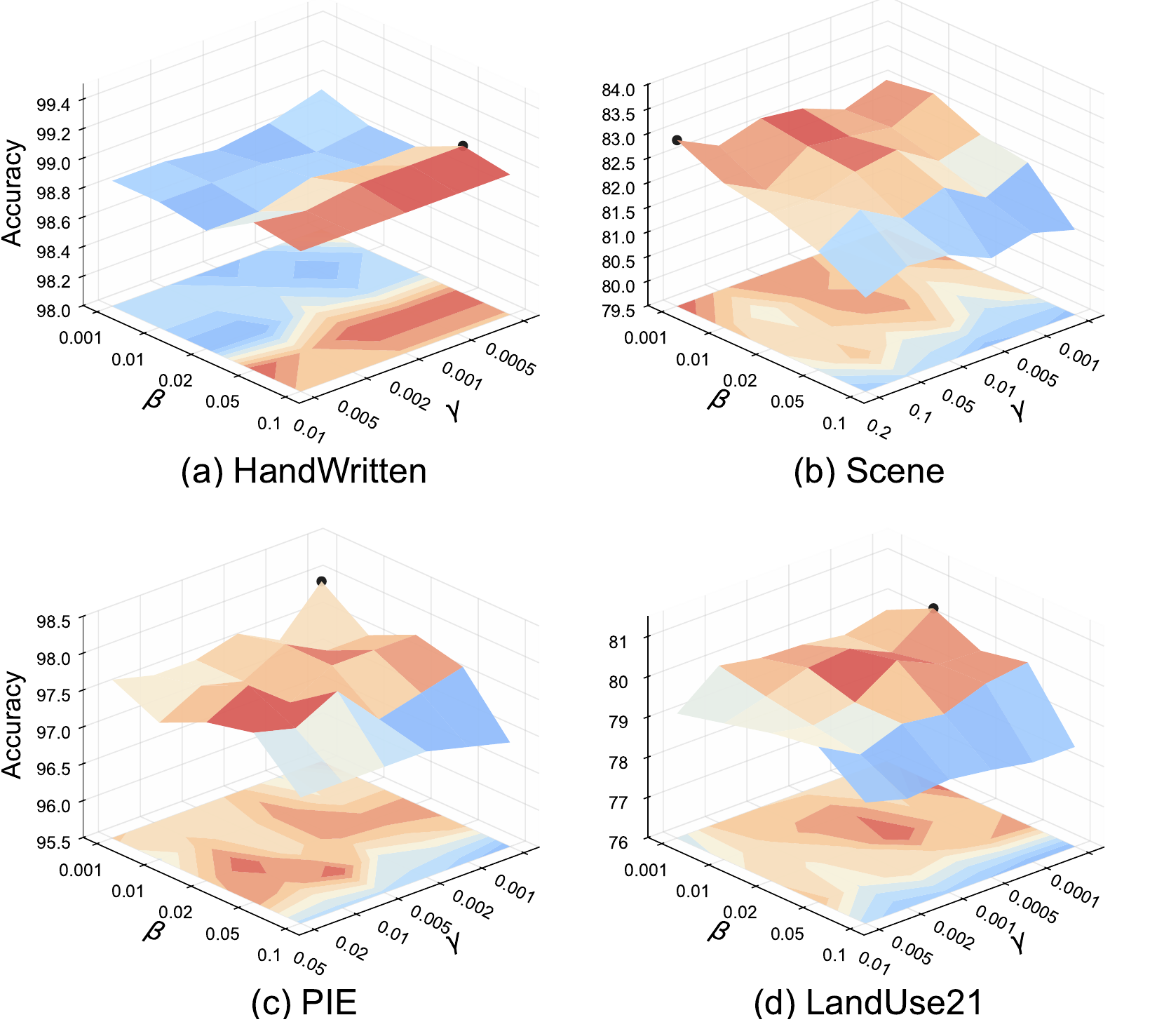}
\caption{Joint sensitivity of TMUR to $\beta$ and $\gamma$ on four representative datasets.}
\Description{Four accuracy surfaces show the joint sensitivity of TMUR to the load-balancing weight beta and diversity weight gamma on HandWritten, Scene, PIE, and LandUse.}
\label{fig:hyperSensitive}
\end{figure}

\subsection{Ablation Study}
Table~\ref{tab:ablation} shows the full TMUR model achieves the best performance on seven of the eight reported datasets and ties for the best result on WebKB, indicating that the three components are complementary rather than redundant. Removing $\mathcal{L}_{\mathrm{bal}}$ leads to visible drops on datasets such as LandUse, PIE, and Leaves, which suggests that the soft load-balancing term helps prevent expert over-concentration while still allowing sample-wise specialization. Removing $\mathcal{L}_{\mathrm{div}}$ degrades performance on Scene, LandUse, Animal, and Caltech-6V, showing that diversity regularization helps the router exploit complementary evidence from different experts. Finally, removing router-side cross-view attention consistently weakens the full model, with the clearest drop appearing on LandUse, which supports our design choice of enhancing the arbitration signal rather than feeding cross-view interaction back into the private experts.

\subsection{Hyperparameter Sensitivity}
We further study the joint sensitivity of the load-balancing weight $\beta$ and the diversity weight $\gamma$ on HandWritten, Scene, PIE, and LandUse. As shown in Fig.~\ref{fig:hyperSensitive}, the accuracy surfaces vary smoothly across a relatively broad hyperparameter region, instead of peaking sharply at a single fragile setting. This indicates that TMUR is reasonably robust to the choices of $\beta$ and $\gamma$, and that its gains do not rely on narrow hyperparameter tuning.

\section{Conclusion}
This paper revisits trustworthy multi-view classification from the perspective of evidence comparability. We argue that heterogeneous views and independently optimized branches need not produce evidence on a common numeric scale, making direct cross-view evidence comparison an unreliable basis for fusion arbitration. To address this issue, we propose TMUR, a unified-routing framework in which view-private experts, a collaborative expert, and a global router play explicitly separated roles. The final prediction remains evidential, while fusion authority is reassigned to router-mediated sample-wise arbitration under the joint multi-view context. Extensive experiments on fourteen datasets, together with calibration and robustness analyses, show that TMUR improves both predictive performance and reliability.

\section*{ACKNOWLEDGMENTS}
This work was supported in part by the National Natural Science Foundation of China under Grants 62425605, 62133012, 62572375, 62303366 and 62472340, in part by Natural Science Basic Research Program of Shaanxi under Grants 2025JC-QYXQ-040, 2025JC-YBQN-900, and in part by Xidian University Specially Funded Project for Interdisciplinary Exploration under Grant TZJHF202506.

\bibliographystyle{ACM-Reference-Format}
\bibliography{reference}

@book{shafer1976mathematical,
  title={A Mathematical Theory of Evidence},
  author={Shafer, Glenn},
  publisher={Princeton University Press},
  year={1976}
}

@article{zadeh1986simple,
  title={A simple view of the Dempster-Shafer theory of evidence and its implication for the rule of combination},
  author={Zadeh, Lotfi A},
  journal={AI magazine},
  volume={7},
  number={2},
  pages={85--85},
  year={1986}
}

@ARTICLE{jacobs1991adaptive,
  author={Jacobs, Robert A. and Jordan, Michael I. and Nowlan, Steven J. and Hinton, Geoffrey E.},
  journal={Neural Computation}, 
  title={Adaptive Mixtures of Local Experts}, 
  year={1991},
  volume={3},
  number={1},
  pages={79-87},
  doi={10.1162/neco.1991.3.1.79}
}

@inproceedings{shazeer2017outrageously,
  title={Outrageously Large Neural Networks: The Sparsely-Gated Mixture-of-Experts Layer},
  author={Noam Shazeer and Azalia Mirhoseini and Krzysztof Maziarz and Andy Davis and Quoc Le and Geoffrey Hinton and Jeff Dean},
  booktitle={International Conference on Learning Representations},
  year={2017},
  url={https://openreview.net/forum?id=B1ckMDqlg}
}

@inproceedings{guo2017calibration,
  title={On calibration of modern neural networks},
  author={Guo, Chuan and Pleiss, Geoff and Sun, Yu and Weinberger, Kilian Q},
  booktitle={Proceedings of the 34th International Conference on Machine Learning},
  pages={1321--1330},
  year={2017},
  organization={PMLR}
}

@article{sensoy2018evidential,
  title={Evidential deep learning to quantify classification uncertainty},
  author={Sensoy, Murat and Kaplan, Lance and Kandemir, Melih},
  journal={Advances in neural information processing systems},
  volume={31},
  year={2018}
}

@inproceedings{han2021trusted,
title={Trusted Multi-View Classification},
author={Zongbo Han and Changqing Zhang and Huazhu Fu and Joey Tianyi Zhou},
booktitle={International Conference on Learning Representations},
year={2021},
url={https://openreview.net/forum?id=OOsR8BzCnl5}
}

@book{josang2018subjective,
author = {J{\o}sang, Audun},
title = {Subjective Logic: A Formalism for Reasoning Under Uncertainty},
year = {2018},
isbn = {3319825550},
publisher = {Springer Publishing Company, Incorporated},
edition = {1st},
}

@article{han2023dynamic,
  title={Trusted multi-view classification with dynamic evidential fusion},
  author={Han, Zongbo and Zhang, Changqing and Fu, Huazhu and Zhou, Joey Tianyi},
  journal={IEEE transactions on pattern analysis and machine intelligence},
  volume={45},
  number={2},
  pages={2551--2566},
  year={2023},
  publisher={IEEE}
}

@inproceedings{liu2022trusted,
  title={Trusted multi-view deep learning with opinion aggregation},
  author={Liu, Wei and Yue, Xiaodong and Chen, Yufei and Denoeux, Thierry},
  booktitle={Proceedings of the AAAI Conference on Artificial Intelligence},
  volume={36},
  number={7},
  pages={7585--7593},
  year={2022}
}

@inproceedings{deng2023uncertainty,
  title={Uncertainty estimation by fisher information-based evidential deep learning},
  author={Deng, Danruo and Chen, Guangyong and Yu, Yang and Liu, Furui and Heng, Pheng-Ann},
  booktitle={International conference on machine learning},
  pages={7596--7616},
  year={2023},
  organization={PMLR}
}

@article{wang2023uncertainty,
  title={Uncertainty-inspired open set learning for retinal anomaly identification},
  author={Wang, Meng and Lin, Tian and Wang, Lianyu and Lin, Aidi and Zou, Ke and Xu, Xinxing and Zhou, Yi and Peng, Yuanyuan and Meng, Qingquan and Qian, Yiming and others},
  journal={Nature Communications},
  volume={14},
  number={1},
  pages={6757},
  year={2023},
  publisher={Nature Publishing Group UK London}
}

@inproceedings{zhang2023provable,
  title={Provable dynamic fusion for low-quality multimodal data},
  author={Zhang, Qingyang and Wu, Haitao and Zhang, Changqing and Hu, Qinghua and Fu, Huazhu and Zhou, Joey Tianyi and Peng, Xi},
  booktitle={International conference on machine learning},
  pages={41753--41769},
  year={2023},
  organization={PMLR}
}

@inproceedings{shirasuna2024moe_molecular,
  title={A multi-view mixture-of-experts based on language and graphs for molecular properties prediction},
  author={Shirasuna, Victor Yukio and Soares, Eduardo and Brazil, Emilio Vital and Gutierrez, Karen Fiorella Aquino and Cerqueira, Renato and Takeda, Seiji and Kishimoto, Akihiro},
  booktitle={ICML 2024 AI for Science Workshop},
  year={2024}
}

@inproceedings{xu2024reliable,
  title={Reliable conflictive multi-view learning},
  author={Xu, Cai and Si, Jiajun and Guan, Ziyu and Zhao, Wei and Wu, Yue and Gao, Xiyue},
  booktitle={Proceedings of the AAAI conference on artificial intelligence},
  volume={38},
  number={14},
  pages={16129--16137},
  year={2024}
}

@ARTICLE{zhang2024AdapRegress,
  author={Zhang, Chenglong and Zhu, Xinjie and Wang, Zidong and Zhong, Yan and Sheng, Weiguo and Ding, Weiping and Jiang, Bingbing},
  journal={IEEE Transactions on Emerging Topics in Computational Intelligence}, 
  title={Discriminative Multi-View Fusion via Adaptive Regression}, 
  year={2024},
  volume={8},
  number={6},
  pages={3821-3833},
  doi={10.1109/TETCI.2024.3375342}
  }

@inproceedings{li2024hyper,
  title={Hyper Evidential Deep Learning to Quantify Composite Classification Uncertainty},
  author={Li, Changbin and Li, Kangshuo and Ou, Yuzhe and Kaplan, Lance M and J{\o}sang, Audun and Cho, Jin-Hee and JEONG, DONG HYUN and Chen, Feng},
  booktitle={The Twelfth International Conference on Learning Representations}
}

@inproceedings{cao2024predictive,
  title={Predictive dynamic fusion},
  author={Cao, Bing and Xia, Yinan and Ding, Yi and Zhang, Changqing and Hu, Qinghua},
  booktitle={Proceedings of the 41st International Conference on Machine Learning},
  pages={5608--5628},
  year={2024}
}

@article{chen2025revisiting,
  title={Revisiting essential and nonessential settings of evidential deep learning},
  author={Chen, Mengyuan and Gao, Junyu and Xu, Changsheng},
  journal={IEEE Transactions on Pattern Analysis and Machine Intelligence},
  year={2025},
  publisher={IEEE}
}

@inproceedings{zhang2025mixture,
  title={Mixture of experts as representation learner for deep multi-view clustering},
  author={Zhang, Yunhe and Cai, Jinyu and Wu, Zhihao and Wang, Pengyang and Ng, See-Kiong},
  booktitle={Proceedings of the AAAI Conference on Artificial Intelligence},
  volume={39},
  number={21},
  pages={22704--22713},
  year={2025}
}

@inproceedings{dong2025trusted,
  title={Trusted Open-World Multi-View Classification with Dynamic Opinion Aggregation},
  author={Dong, Zhicheng and Yue, Xiaodong and Chen, Yufei and Zhou, Yuxian},
  booktitle={Proceedings of the 33rd ACM International Conference on Multimedia},
  pages={1181--1189},
  year={2025}
}

@inproceedings{ding2025theoretical,
  title={A Theoretical Proof of Dynamic Multimodal Fusion Exacerbates Modality Greedy},
  author={Ding, Xiaorui and Ma, Huan and Zhang, Changqing},
  booktitle={Proceedings of the 33rd ACM International Conference on Multimedia},
  pages={2429--2436},
  year={2025}
}

@inproceedings{xu2025noisy,
  title={Noisy label calibration for multi-view classification},
  author={Xu, Shilin and Sun, Yuan and Li, Xingfeng and Duan, Siyuan and Ren, Zhenwen and Liu, Zheng and Peng, Dezhong},
  booktitle={Proceedings of the AAAI Conference on Artificial Intelligence},
  volume={39},
  number={20},
  pages={21797--21805},
  year={2025}
}

@inproceedings{liu2025enhancingAAAI,
  title={Enhancing multi-view classification reliability with adaptive rejection},
  author={Liu, Wei and Chen, Yufei and Yue, Xiaodong},
  booktitle={Proceedings of the AAAI Conference on Artificial Intelligence},
  volume={39},
  number={18},
  pages={18969--18977},
  year={2025}
}

@article{liu2025enhancingTCSVT,
  title={Enhancing reliability in medical image classification of imperfect views},
  author={Liu, Wei and Chen, Yufei and Yue, Xiaodong and Zhang, Changqing and Xie, Shaorong},
  journal={IEEE Transactions on Circuits and Systems for Video Technology},
  year={2026},
  volume={36},
  number={2},
  pages={2290-2303}
}

@inproceedings{lin2025enhance,
  title={Enhance multi-view classification through multi-scale alignment and expanded boundary},
  author={Lin, Yuena and Wang, Yiyuan and Lyu, Gengyu and Deng, Yongjian and Cai, Haichun and Lin, Huibin and Wang, Haobo and Yang, Zhen},
  booktitle={The Thirteenth International Conference on Learning Representations},
  year={2025}
}

@inproceedings{zhou2025refining,
  title={Refining Confusion and Ignorance in Trusted Multi-View Classification},
  author={Zhou, Xujing and Yue, Xiaodong and Chen, Yufei and Li, Linye},
  booktitle={Companion Proceedings of the ACM on Web Conference 2025},
  pages={1549--1553},
  year={2025}
}

@inproceedings{fang2025enhancing,
  title={Enhancing Multi-view Open-set Learning via Ambiguity Uncertainty Calibration and View-wise Debiasing},
  author={Fang, Zihan and Xu, Zhiyong and Du, Lan and Du, Shide and Cai, Zhiling and Wang, Shiping},
  booktitle={Proceedings of the 33rd ACM International Conference on Multimedia},
  pages={1220--1228},
  year={2025}
}

@inproceedings{huang2025trusted,
  title={Trusted unified feature-neighborhood dynamics for multi-view classification},
  author={Huang, Haojian and Qin, Chuanyu and Liu, Zhe and Ma, Kaijing and Chen, Jin and Fang, Han and Ban, Chao and Sun, Hao and He, Zhongjiang},
  booktitle={Proceedings of the AAAI conference on artificial intelligence},
  volume={39},
  number={16},
  pages={17413--17421},
  year={2025}
}

@inproceedings{lu2025navigating,
  title={Navigating Conflicting Views: Harnessing Trust for Learning},
  author={Lu, Jueqing and Buntine, Wray and Qi, Yuanyuan and Dipnall, Joanna and Gabbe, Belinda and Du, Lan},
  booktitle={Proceedings of the 42nd International Conference on Machine Learning},
  publisher={PMLR},
  year={2025},
}

@inproceedings{lan2025BCM,
  title={Multi-view Hashing Classification},
  author={Lan, Yuhang and Xu, Shilin and Su, Chao and Ye, Run and Peng, Dezhong and Sun, Yuan},
  booktitle={Proceedings of the 33rd ACM International Conference on Multimedia},
  pages={2122--2130},
  year={2025}
}

@inproceedings{duan2025fuml,
  title={Deep fuzzy multi-view learning for reliable classification},
  author={Duan, Siyuan and Sun, Yuan and Peng, Dezhong and Duan, Guiduo and Peng, Xi and Hu, Peng},
  booktitle={Proceedings of the 42nd International Conference on Machine Learning},
  publisher={PMLR},
  year={2025}
}

@inproceedings{liang2025expert,
  title={Trusted multi-view classification with expert knowledge constraints},
  author={Liang, Xinyan and Wang, Shijie and Qian, Yuhua and Guo, Qian and Du, Liang and Jiang, Bingbing and Luo, Tingjin and Li, Feijiang},
  booktitle={Forty-second International Conference on Machine Learning},
  publisher={PMLR},
  year={2025}
}

@inproceedings{liang2025evolutionary,
  title={Trusted multi-view classification via evolutionary multi-view fusion},
  author={Liang, Xinyan and Fu, Pinhan and Qian, Yuhua and Guo, Qian and Liu, Guoqing},
  booktitle={The Thirteenth International Conference on Learning Representations},
  year={2025}
}

@inproceedings{wu2025mvcge,
  title={Where Graph Meets Heterogeneity: Multi-View Collaborative Graph Experts},
  author={Wu, Zhihao and Cai, Jinyu and Zhang, Yunhe and Lu, Jielong and Chen, Zhaoliang and Zhuang, Shuman and Wang, Haishuai},
  booktitle={Advances in Neural Information Processing Systems},
  year={2025}
}

@inproceedings{xu2025beyond,
  title={Beyond Equal Views: Strength-Adaptive Evidential Multi-View Learning},
  author={Xu, Cai and Wen, Ziqi and Zhao, Jie and Zhao, Wanqing and Yu, Jinlong and Chen, Haishun and Guan, Ziyu and Zhao, Wei},
  booktitle={Proceedings of the 33rd ACM International Conference on Multimedia},
  pages={1278--1287},
  year={2025}
}

@inproceedings{hu2025self,
  title={Self-supervised trusted contrastive multi-view clustering with uncertainty refined},
  author={Hu, Shizhe and Tian, Binyan and Liu, Weibo and Ye, Yangdong},
  booktitle={Proceedings of the AAAI Conference on Artificial Intelligence},
  volume={39},
  number={16},
  pages={17305--17313},
  year={2025}
}

@inproceedings{I2MoE,
  title={I2MoE: Interpretable Multimodal Interaction-aware Mixture-of-Experts},
  author={Xin, Jiayi and Yun, Sukwon and Peng, Jie and Choi, Inyoung and Ballard, Jenna L and Chen, Tianlong and Long, Qi},
  booktitle={Proceedings of the 42nd International Conference on Machine Learning},
  year={2025},
  organization={PMLR}
}

@article{hu2026robust,
  title={Robust Trusted Conflictive Multiview Collaborative Contrastive Learning},
  author={Hu, Shaobo and Huang, Hui and Zhang, Nan and Sun, Shiliang},
  journal={IEEE Transactions on Pattern Analysis and Machine Intelligence},
  year={2026},
  publisher={IEEE}
}

@inproceedings{shi2026not,
  title={Not All Inconsistency Is Equal: Decomposing LVLM Uncertainty into Belief Divergence and Belief Conflict},
  author={Shi, Jie and Yue, Xiaodong and Liu, Wei and Chen, Yufei and Dong, Feifan},
  booktitle={Proceedings of the AAAI Conference on Artificial Intelligence},
  volume={40},
  number={30},
  pages={25339--25347},
  year={2026}
}

@inproceedings{qin2026multi,
  title={Multi-view Learning via Trusted Pairwise Entity Energy},
  author={Qin, Yalan and Feng, Guorui and Zhang, Xinpeng},
  booktitle={Proceedings of the AAAI Conference on Artificial Intelligence},
  volume={40},
  number={29},
  pages={24954--24962},
  year={2026}
}

@article{Huang2026M2FMoE, 
title={M2FMoE: Multi-Resolution Multi-View Frequency Mixture-of-Experts for Extreme-Adaptive Time Series Forecasting}, 
volume={40}, 
DOI={10.1609/aaai.v40i26.39362},  
number={26}, 
journal={Proceedings of the AAAI Conference on Artificial Intelligence}, 
author={Huang, Yaohui and Zou, Runmin and Wang, Yun and Aslam, Laeeq and Dong, Ruipeng}, 
year={2026}, 
pages={22075-22083} }

@ARTICLE{zhang2026NoisySup,
  author={Zhang, Yilin and Xu, Cai and Jiang, Han and Guan, Ziyu and Zhao, Wei and He, Xiaofei and Sensoy, Murat},
  journal={IEEE Transactions on Pattern Analysis and Machine Intelligence},
  title={Trusted Multi-View Learning Under Noisy Supervision},
  year={2026},
  volume={48},
  number={9},
  pages={10827-10843},
  doi={10.1109/TPAMI.2026.3690466}
}

@inproceedings{liu2026beyond,
title={Beyond Magnitude: Scale-Invariant Evidential Fusion for Multi-View Classification},
author={Wei Liu and Yufei Chen and Jie Shi and Xiaodong Yue},
booktitle={Forty-third International Conference on Machine Learning},
year={2026},
url={https://openreview.net/forum?id=ZZNKx9eZMB}
}

@inproceedings{zhu2026online,
  title={Online Multi-Relational Clustering with Dominant View Mining},
  author={Zhu, Zhengzhong and Zhou, Pei and Wang, Dongsheng and Cheng, Li and Zhu, Jiangping},
  booktitle={Proceedings of the AAAI Conference on Artificial Intelligence},
  volume={40},
  number={34},
  pages={29232--29240},
  year={2026},
  doi={10.1609/aaai.v40i34.40162}
}

@inproceedings{xue2026grok,
  title={{GRPO}-based Cluster Decision Agent for Unknown-$K$ Multi-view Clustering},
  author={Xue, Xuqian and Zhang, Jun and Cai, Qi and Huang, Zhizhong and Shan, Hongming and Zhang, Junping},
  booktitle={Proceedings of the 43rd International Conference on Machine Learning},
  year={2026},
  url={https://icml.cc/virtual/2026/poster/62944}
}

@article{li2023crossview,
title = {Cross-view Graph Matching Guided Anchor Alignment for Incomplete Multi-view Clustering},
journal = {Information Fusion},
volume = {100},
pages = {101941},
year = {2023},
issn = {1566-2535},
doi = {https://doi.org/10.1016/j.inffus.2023.101941},
url = {https://www.sciencedirect.com/science/article/pii/S1566253523002579},
author = {Xingfeng Li and Yinghui Sun and Quansen Sun and Zhenwen Ren and Yuan Sun},
}

@ARTICLE{yuan2025prototype,
  author={Yuan, Honglin and Sun, Yuan and Zhou, Fei and Wen, Jing and Yuan, Shihua and You, Xiaojian and Ren, Zhenwen},
  journal={IEEE Transactions on Image Processing}, 
  title={Prototype Matching Learning for Incomplete Multi-View Clustering}, 
  year={2025},
  volume={34},
  number={},
  pages={828-841},
  doi={10.1109/TIP.2025.3529378}
}

\balance
\clearpage
\nobalance
\makeatletter
\@ACM@balancefalse
\makeatother
\appendix
\setcounter{equation}{0}
\setcounter{table}{0}
\setcounter{figure}{0}
\renewcommand{\theequation}{A\arabic{equation}}
\renewcommand{\thetable}{A\arabic{table}}
\renewcommand{\thefigure}{A\arabic{figure}}
\renewcommand{\theHequation}{appendix.\arabic{equation}}
\renewcommand{\theHtable}{appendix.\arabic{table}}
\renewcommand{\theHfigure}{appendix.\arabic{figure}}

\section{Overview of the Appendix}
\label{sec:appendix-overview}
This appendix includes the following supplementary materials:
\begin{itemize}
    \item algorithmic details of TMUR, including the complete training and inference procedure;
    \item dataset descriptions and the dataset-specific hyperparameter settings used in our benchmark;
    \item additional accuracy results, calibration diagnostics, and hyperparameter sensitivity results;
    \item per-dataset routing diagnostics, larger-dataset and multimodal sentiment results, and computational cost;
    \item cross-view calibration plots and supplementary theoretical derivations.
\end{itemize}

\section{Algorithmic Description of TMUR}
Algorithm~\ref{alg:tmur} summarizes the training and inference procedure of TMUR. 

\begin{algorithm}[ht]
    \caption{Training and inference procedure of TMUR.}
    \label{alg:tmur}
    \raggedright
    \textbf{/* TRAIN */}\\
    \textbf{Input}: Multi-view training set $\mathcal{D}$, epochs $T$, hyperparameters $\lambda$, $\beta$, $\gamma$, routing temperature $\tau$.\\
    \textbf{Output}: Parameters of view projectors, experts, and router.
    \begin{algorithmic}[1]
        \State Initialize view projectors $\{P_v\}_{v=1}^{V}$, private experts $\{E_v\}_{v=1}^{V}$, collaborative expert $E_{V+1}$, and unified router $R$.
        \For{epoch $t = 1 : T$}
            \For{mini-batch $\mathcal{B}\subset\mathcal{D}$}
                \State Align each view by $\bm{h}^{(v)} = \mathrm{LN}(P_v(\bm{x}^{(v)}))$ for $v=1,\ldots,V$.
                \State Form the collaborative input $\bm{H} = [\bm{h}^{(1)};\cdots;\bm{h}^{(V)}]$.
                \State Build router context $\bm{g}$ by router-side cross-view attention over $\{\bm{h}^{(v)}\}_{v=1}^{V}$.
                \State Obtain private expert evidence $\bm{e}^{(v)} = \phi(E_v(\bm{h}^{(v)}))$ for $v=1,\ldots,V$.
                \State Obtain collaborative evidence $\bm{e}^{(V+1)} = \phi(E_{V+1}(\bm{H}))$.
                \State Predict sample-wise weights $\bm{\pi}=\mathrm{softmax}(R(\bm{g})/\tau)$.
                \State Aggregate evidence $\bm{e}^{\mathrm{fused}}=\sum_{i=1}^{V+1} \pi_i \bm{e}^{(i)}$ and set $\bm{\alpha}^{\mathrm{fused}}=\bm{e}^{\mathrm{fused}}+\bm{1}$.
                \State Compute fused loss $\mathcal{L}_{\mathrm{fused}}$ and auxiliary expert loss $\mathcal{L}_{\mathrm{view}}$.
                \State Compute load-balancing loss $\mathcal{L}_{\mathrm{bal}}$ and expert-diversity loss $\mathcal{L}_{\mathrm{div}}$.
                \State Update all parameters using
                \[
                \mathcal{L}
                =
                \mathcal{L}_{\mathrm{fused}}
                + \lambda \mathcal{L}_{\mathrm{view}}
                + \beta \mathcal{L}_{\mathrm{bal}}
                + \gamma \mathcal{L}_{\mathrm{div}}.
                \]
            \EndFor
        \EndFor
    \end{algorithmic}
    \textbf{/* TEST */}\\
    For each test sample, repeat the forward pass and return $\bm{p}^{\mathrm{fused}}=\bm{\alpha}^{\mathrm{fused}}/S^{\mathrm{fused}}$, $u^{\mathrm{fused}}=K/S^{\mathrm{fused}}$, and $\arg\max_k p_k^{\mathrm{fused}}$, where $S^{\mathrm{fused}}=\sum_k\alpha_k^{\mathrm{fused}}$.
\end{algorithm}

\section{Dataset Details}
\subsection{Benchmark Statistics}
Table~\ref{tab:appendix-dataset-stats} reports the detailed statistics of all fourteen datasets used in the main paper. We follow the standard processed multi-view releases adopted by prior trusted multi-view studies. For some datasets, such as Caltech101 and CUB, the benchmark release itself uses the commonly adopted subset version rather than the original full dataset, which explains the reduced sample and class counts.

\begin{table*}[!t]
\centering
\small
\caption{Detailed statistics of the fourteen datasets used in the main paper.}
\label{tab:appendix-dataset-stats}
\resizebox{0.7\textwidth}{!}{
\begin{tabular}{lcccl}
\toprule
\textbf{Dataset} & \textbf{\#Samples} & \textbf{\#Views} & \textbf{\#Classes} & \textbf{Per-view Dimensions}\\
\midrule
HandWritten & 2000 & 6 & 10 & \{240, 76, 216, 47, 64, 6\}\\
Scene & 4485 & 3 & 15 & \{20, 59, 40\}\\
LandUse & 2100 & 3 & 21 & \{20, 59, 40\}\\
NUS & 30000 & 5 & 31 & \{65, 226, 145, 74, 129\}\\
Caltech-6V & 2386 & 6 & 20 & \{48, 40, 254, 1984, 512, 928\}\\
PIE & 680 & 3 & 68 & \{484, 256, 279\}\\
WebKB & 203 & 3 & 4 & \{1703, 230, 230\}\\
Caltech & 2908 & 2 & 10 & \{4096, 4096\}\\
UCI & 2000 & 3 & 10 & \{6, 240, 47\}\\
CUB & 600 & 2 & 10 & \{1024, 300\}\\
Animal & 10158 & 2 & 50 & \{4096, 4096\}\\
MSRCV1 & 210 & 6 & 7 & \{1302, 48, 512, 100, 256, 210\}\\
BBC & 685 & 4 & 5 & \{4659, 4633, 4665, 4684\}\\
Leaves & 1600 & 3 & 100 & \{64, 64, 64\}\\
\bottomrule
\end{tabular}}
\end{table*}

\begin{table*}[t]
\centering
\small
\caption{Dataset-specific TMUR hyperparameters used in the benchmark experiments.}
\label{tab:appendix-hparams}
\resizebox{0.55\textwidth}{!}{
\begin{tabular}{lcc|lcc}
\toprule
\textbf{Dataset} & $\bm{\beta}$ & $\bm{\gamma}$ & \textbf{Dataset} & $\bm{\beta}$ & $\bm{\gamma}$\\
\midrule
HandWritten & 0.1 & 0.0 & Scene & 0.0 & 0.2\\
LandUse & 0.05 & 0.0001 & NUS & 0.05 & 0.05\\
Caltech-6V & 0.05 & 0.05 & PIE & 0.02 & 0.02\\
WebKB & 0.05 & 0.01 & Caltech & 0.05 & 0.1\\
UCI & 0.05 & 0.1 & CUB & 0.05 & 0.05\\
Animal & 0.01 & 0.02 & MSRCV1 & 0.05 & 0.01\\
BBC & 0.0 & 0.05 & Leaves & 0.01 & 0.002\\
\bottomrule
\end{tabular}}
\end{table*}

\subsection{Per-dataset View Descriptions}
\noindent\textbf{HandWritten.} HandWritten contains 2000 handwritten numerals from ``0'' to ``9'', with 200 samples per class. Following the standard benchmark release, each sample is represented by six views: pixel averages in $2\times3$ windows, Fourier coefficients of the character shape, profile correlations, Zernike moments, Karhunen--Loeve coefficients, and morphological features.

\noindent\textbf{Scene.} Scene15 contains 4485 indoor and outdoor scene images from 15 categories. Its three views are GIST, PHOG, and LBP descriptors extracted from the same image.

\noindent\textbf{LandUse.} LandUse21 contains 2100 remote-sensing images from 21 land-use categories. In the benchmark release used in this paper, each aerial image is represented by three visual views with dimensions 20, 59, and 40.

\noindent\textbf{NUS.} NUS-WIDE-Object contains 30000 images from 31 classes. Each image is described by five visual views, namely color histogram, block-wise color moments, color correlogram, edge direction histogram, and wavelet texture.

\noindent\textbf{Caltech-6V.} Caltech101-6V is the standard six-view benchmark subset derived from Caltech101. Each image is represented by six complementary visual views with dimensions 48, 40, 254, 1984, 512, and 928.

\noindent\textbf{PIE.} PIE contains 680 face images from 68 subjects. The three views are intensity features, LBP features, and Gabor features extracted from the same face image.

\noindent\textbf{WebKB.} WebKB contains 203 webpages from four categories. Each page is described by three text views: the page content, the anchor text of hyperlinks, and the title text.

\noindent\textbf{Caltech.} The 2-view Caltech101 benchmark selects the top 10 categories from the original 101 classes and uses DECAF and VGG19 to extract two 4096-dimensional image views.

\noindent\textbf{UCI.} UCI contains handwritten numerals from 10 classes. The three views are the average of pixels in 240 windows, 47 Zernike moments, and 6 morphological features.

\noindent\textbf{CUB.} CUB is constructed from the first 10 bird categories of the original dataset. Following the standard benchmark setting, GoogleNet features are used as the image view and doc2vec features are used as the text view.

\noindent\textbf{Animal.} Animal contains 10158 images from 50 animal categories. The two views are deep image features extracted by DECAF and VGG19.

\noindent\textbf{MSRCV1.} MSRCV1 contains 210 images from 7 classes. The processed version used in this paper provides six visual views for each image, with dimensions 1302, 48, 512, 100, 256, and 210.

\noindent\textbf{BBC.} BBC contains 685 news documents from 5 categories. Each document is represented by four text views in the standard multi-view benchmark release.

\noindent\textbf{Leaves.} Leaves100 contains 1600 leaf samples from 100 plant species. Its three views correspond to shape descriptors, fine-scale edge features, and texture histograms.

\section{Additional Experimental Details}
\subsection{Training Protocol}
The fourteen-dataset benchmark uses 300 training epochs and five fixed random seeds $\{3407, 7, 601, 101, 503\}$. The train/test partitions follow the stratified 80/20 split implemented in the released code. Optimizer, learning-rate schedule, batch size, and auxiliary-supervision weight follow Section~5.1. The calibrated/normalized local-fusion comparison in Section~5.3 uses its separately specified 70/10/20 train/validation/test split.

\subsection{Dataset-specific Hyperparameter Settings}
Table~\ref{tab:appendix-hparams} lists the dataset-specific load-balancing weight $\beta$ and diversity weight $\gamma$ for the benchmark.

\subsection{Calibration Metrics}
For a test sample $n$, let $a_n=\mathbb{1}[\arg\max_k p_{n,k}=y_n]$. Probability ECE uses the confidence $c_n=\max_k p_{n,k}$, whereas uncertainty ECE (U-ECE) uses $c_n=1-u_n$, with $u_n=K/\sum_k\alpha_{n,k}$ derived from the Dirichlet opinion. For equal-width bins $B_1,\ldots,B_M$ of $c_n$, the calibration error is
\begin{equation}
\mathrm{ECE}=\sum_{m:|B_m|>0}\frac{|B_m|}{N}
\left|\frac{1}{|B_m|}\sum_{n\in B_m}a_n-
\frac{1}{|B_m|}\sum_{n\in B_m}c_n\right|.
\end{equation}
The ECE comparison in Table~\ref{tab:calibration_ece} uses uncertainty-derived confidence. The archived diagnostic runs summarized below use $M=15$ equal-width bins. Accuracy and ECE are expressed as percentages; mean uncertainty is on the $[0,1]$ scale.

\begin{table*}[t]
\centering
\small
\caption{Supplementary five-seed accuracy results on the remaining seven datasets that are omitted from the main-paper benchmark table for space. Each entry reports mean $\pm$ standard deviation over five seeds.}
\label{tab:appendix-acc-back}
\resizebox{\textwidth}{!}{
\begin{tabular}{l|l|ccccccc}
\toprule
\textbf{Method} & \textbf{Ref.} & \textbf{Caltech} & \textbf{UCI} & \textbf{CUB} & \textbf{Animal} & \textbf{MSRCV1} & \textbf{BBC} & \textbf{Leaves} \\
\midrule
TMC~\cite{han2021trusted} & ICLR'21 & 99.52$\pm$0.07 & 97.90$\pm$1.04 & 92.83$\pm$1.87 & 87.48$\pm$0.28 & 95.24$\pm$3.69 & 94.45$\pm$1.19 & 96.13$\pm$0.81 \\
TMDLO~\cite{liu2022trusted} & AAAI'22 & 83.37$\pm$8.99 & 91.60$\pm$6.41 & 85.17$\pm$7.18 & 75.37$\pm$2.01 & 95.24$\pm$3.37 & 94.16$\pm$1.85 & 70.25$\pm$2.18 \\
ETMC~\cite{han2023dynamic} & PAMI'23 & 99.59$\pm$0.14 & 98.10$\pm$0.75 & 93.67$\pm$1.72 & \underline{88.30$\pm$0.43} & 98.10$\pm$1.78 & 94.31$\pm$1.49 & 98.88$\pm$0.25 \\
RCML~\cite{xu2024reliable} & AAAI'24 & 96.12$\pm$0.84 & 98.35$\pm$0.64 & 93.00$\pm$1.72 & 87.57$\pm$0.53 & 90.95$\pm$1.78 & 93.43$\pm$1.22 & 95.63$\pm$0.97 \\
ETF~\cite{lu2025navigating} & ICML'25 & 99.11$\pm$0.48 & 96.90$\pm$1.11 & 92.33$\pm$2.07 & \textbf{88.40$\pm$0.29} & 93.81$\pm$1.17 & 92.12$\pm$1.26 & 42.94$\pm$7.41 \\
FUML~\cite{duan2025fuml} & ICML'25 & \underline{99.76$\pm$0.08} & 98.80$\pm$0.90 & 93.33$\pm$1.90 & 87.21$\pm$1.06 & 98.10$\pm$2.33 & 79.71$\pm$13.28 & \textbf{99.88$\pm$0.15} \\
TMCEK~\cite{liang2025expert} & ICML'25 & 74.19$\pm$16.97 & 98.75$\pm$0.65 & 94.00$\pm$1.93 & 88.19$\pm$0.53 & 91.43$\pm$2.43 & 96.35$\pm$0.00 & 97.81$\pm$0.66 \\
TUNED~\cite{huang2025trusted} & AAAI'25 & 99.59$\pm$0.18 & 98.90$\pm$0.78 & 94.67$\pm$1.25 & 82.41$\pm$1.33 & \underline{99.05$\pm$1.17} & \underline{96.93$\pm$1.17} & 94.25$\pm$1.00 \\
TEF~\cite{liang2025evolutionary} & ICLR'25 & \underline{99.76$\pm$0.08} & \underline{99.00$\pm$0.57} & \underline{94.83$\pm$1.33} & 87.78$\pm$0.53 & \underline{99.05$\pm$1.17} & 96.79$\pm$0.99 & 89.69$\pm$1.15 \\
SAEML~\cite{xu2025beyond} & MM'25 & 87.18$\pm$5.11 & 98.90$\pm$0.51 & 93.83$\pm$1.55 & 87.42$\pm$0.51 & \textbf{99.52$\pm$0.95} & \textbf{97.23$\pm$0.55} & 97.75$\pm$0.85 \\
RTMC~\cite{zhou2025refining} & WWW'25 & \textbf{99.79$\pm$0.13} & 96.15$\pm$1.32 & 94.00$\pm$0.82 & 85.80$\pm$0.94 & 80.95$\pm$6.02 & 93.14$\pm$1.76 & 85.94$\pm$2.55 \\
RCMCL~\cite{hu2026robust} & PAMI'26 & 91.92$\pm$0.52 & 97.40$\pm$1.10 & 94.17$\pm$1.83 & 87.97$\pm$0.38 & 97.14$\pm$2.33 & 94.16$\pm$1.31 & 97.12$\pm$0.61 \\
\midrule
NLC~\cite{xu2025noisy} & AAAI'25 & 98.49$\pm$1.28 & 98.50$\pm$0.74 & 87.83$\pm$2.92 & 85.92$\pm$0.56 & 70.48$\pm$6.14 & 78.39$\pm$5.47 & 99.44$\pm$0.41 \\
MAMC~\cite{lin2025enhance} & ICLR'25 & 95.22$\pm$0.61 & 98.80$\pm$0.58 & 94.50$\pm$0.67 & 86.23$\pm$0.73 & 93.81$\pm$5.13 & 95.47$\pm$1.49 & 99.56$\pm$0.15 \\
BCM~\cite{lan2025BCM} & MM'25 & 99.52$\pm$0.25 & \textbf{99.10$\pm$0.46} & 93.83$\pm$0.85 & 87.20$\pm$0.64 & 98.10$\pm$2.33 & 91.82$\pm$0.72 & \underline{99.69$\pm$0.34} \\
\midrule
\rowcolor[RGB]{234, 238, 234}
TMUR & Ours & \textbf{99.79$\pm$0.17} & 98.90$\pm$0.46 & \textbf{95.83$\pm$1.75} & 87.83$\pm$0.57 & \textbf{99.52$\pm$0.95} & 95.04$\pm$1.26 & \underline{99.69$\pm$0.20} \\
\bottomrule
\end{tabular}}
\end{table*}

\begin{table}[t]
\centering
\small
\caption{TMUR diagnostics from the archived five-seed calibration runs. Accuracy and both ECE metrics are percentages; $\overline{u}$ denotes mean uncertainty.}
\label{tab:appendix-ece-detail}
\setlength{\tabcolsep}{4pt}
\begin{tabular}{lrrrr}
\toprule
\textbf{Dataset} & \textbf{Acc.} & \textbf{Prob-ECE} & \textbf{U-ECE} & $\overline{u}$\\
\midrule
HandWritten & 99.05 & 6.19 & 5.36 & 0.059\\
Scene & 82.36 & 5.67 & 6.87 & 0.131\\
LandUse & 79.14 & 14.32 & 7.35 & 0.254\\
Caltech-6V & 96.65 & 10.47 & 7.55 & 0.104\\
PIE & 97.65 & 58.81 & 52.50 & 0.547\\
WebKB & 88.29 & 16.74 & 14.01 & 0.220\\
Animal & 87.81 & 14.34 & 9.79 & 0.192\\
Leaves & 99.69 & 52.15 & 44.43 & 0.447\\
\bottomrule
\end{tabular}
\end{table}

\section{Supplementary Results}
\subsection{Accuracy on the Remaining Seven Datasets}
To keep the main paper readable, Table~\ref{tab:appendix-acc-back} reports the remaining seven accuracy columns that are not shown in Table~1 of the main paper. The same five-seed evaluation protocol and baseline grouping are used here.

TMUR achieves the best or tied-best mean accuracy on Caltech, CUB, and MSRCV1. On CUB, its 95.83\% accuracy exceeds the strongest compared baseline by 1.00 percentage point. On Leaves, it ties BCM for the second-best mean accuracy, below FUML. The table also retains the results on UCI, Animal, and BBC, where TMUR does not obtain the highest mean accuracy.

\subsection{Detailed Calibration Summary for TMUR}
Table~\ref{tab:appendix-ece-detail} reports probability ECE, uncertainty ECE, mean predictive uncertainty, and the corresponding run-specific accuracy for the eight calibration datasets. All quantities in this table are retained from the same archived calibration runs, enabling direct comparison of the two ECE metrics. These run-specific diagnostics supplement the published benchmark comparison in Table~\ref{tab:calibration_ece}.

PIE and Leaves have high accuracy but large probability ECE and mean uncertainty. This illustrates that predictive correctness and calibration measure different properties; strong accuracy alone does not ensure calibrated predictions.

\subsection{Additional Hyperparameter Sensitivity}
Figure~\ref{fig:hyperSensitiveAppendix} reports eight supplementary sensitivity surfaces, and Figure~\ref{fig:hyperSensitiveAdditionalTwo} reports NUS and Animal. Together with the four datasets in Figure~\ref{fig:hyperSensitive}, these results cover all fourteen benchmark datasets. For MSRCV1, accuracy remains close to 1.0 across the displayed range. Accuracy is shown as a percentage, with dataset-specific axis ranges.

\begin{figure*}[!tp]
\centering
\includegraphics[width=\textwidth]{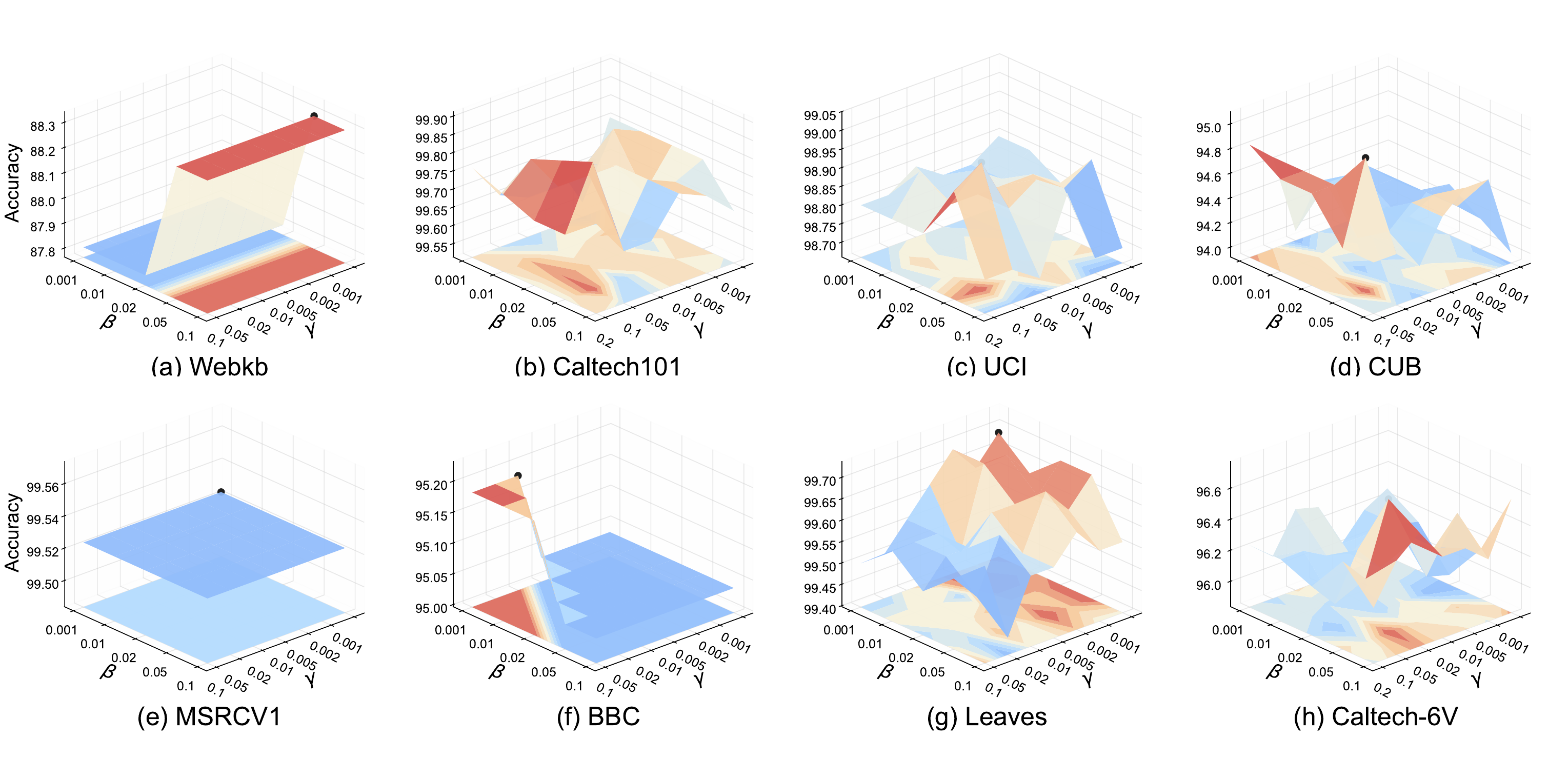}
\caption{Joint $\beta$--$\gamma$ sensitivity on eight additional datasets: WebKB, Caltech, UCI, CUB, MSRCV1, BBC, Leaves, and Caltech-6V. Each subplot shows five-seed accuracy under the corresponding two-dimensional sweep. Caltech101 denotes the two-view Caltech dataset.}
\Description{Eight original sensitivity surfaces arranged in two rows and four columns. The top row contains WebKB, Caltech, UCI, and CUB; the bottom row contains MSRCV1, BBC, Leaves, and Caltech-6V.}
\label{fig:hyperSensitiveAppendix}
\end{figure*}

\begin{figure}[t]
\centering
\includegraphics[width=\columnwidth]{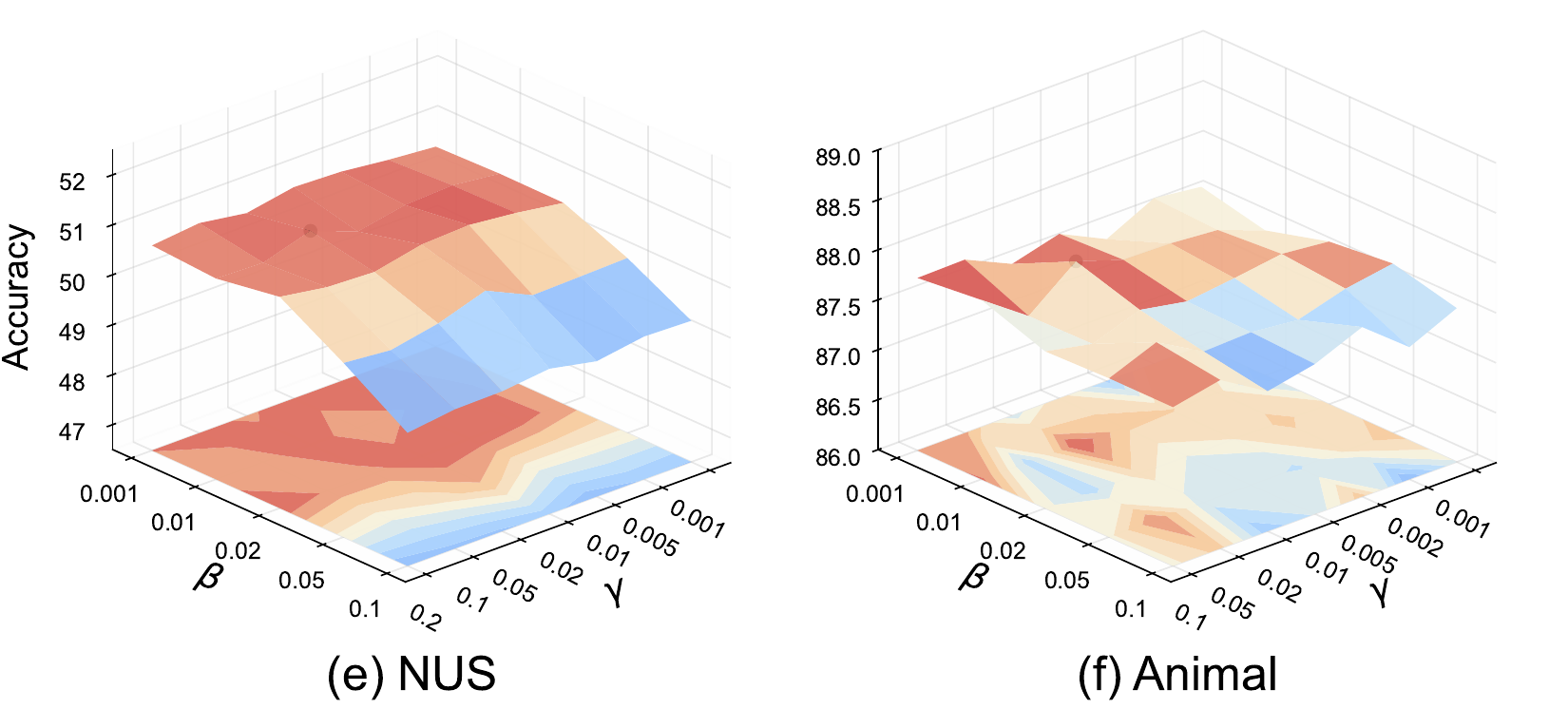}
\caption{Additional joint sensitivity surfaces for NUS and Animal, preserving panels (e) and (f) from the original six-dataset sweep.}
\Description{Two original accuracy surfaces show the sensitivity to beta and gamma on NUS and Animal.}
\label{fig:hyperSensitiveAdditionalTwo}
\end{figure}

\subsection{Per-dataset Router--Reliability Alignment}
Section~5.3 reports the average alignment between routing weights and expert reliability. Table~\ref{tab:appendix-router-behavior} provides the corresponding per-dataset results under the same clean-trained setting, averaged over five seeds.

For sample $n$ and expert $i$, let $\pi_{n,i}$ be its routing weight and $a_{n,i}$ indicate whether its standalone prediction is correct. The pooled statistic $\rho_s$ is the Pearson correlation between $\pi_{n,i}$ and $a_{n,i}$ over all $N(V+1)$ sample--expert pairs. At the expert level, define mean utilization $\bar\pi_i=N^{-1}\sum_n\pi_{n,i}$ and standalone accuracy $A_i=N^{-1}\sum_n a_{n,i}$. The statistic $\rho_e$ is the Pearson correlation between $\bar\pi_i$ and $A_i$ across the $V+1$ experts; Spearman measures the same relation after ranking the two quantities. Each statistic is computed within a run before averaging over seeds.

\begin{table}[ht]
\centering
\small
\caption{Per-dataset router diagnostics averaged over five seeds.}
\label{tab:appendix-router-behavior}
\setlength{\tabcolsep}{10pt}
\begin{tabular}{lrrr}
\toprule
Dataset & $\rho_s$ & $\rho_e$ & Spearman \\
\midrule
Caltech-6V & 0.145 & 0.756 & 0.827 \\
LandUse & 0.216 & 0.863 & 0.800 \\
Scene & 0.162 & 0.744 & 0.600 \\
WebKB & 0.191 & 0.742 & 0.950 \\
\bottomrule
\end{tabular}
\end{table}

The pooled correlations are positive on all four datasets, and the expert-level correlations indicate that experts with higher overall accuracy tend to receive greater mean utilization. These are descriptive associations: the pooled statistic combines within-expert and between-expert variation and does not establish correct expert selection for every individual sample.

\subsection{Evaluation on Larger Multi-view Datasets}
We evaluate four methods on YouTubeFace and MVoxCeleb (Table~\ref{tab:appendix-large-datasets}). TMUR obtains the highest mean accuracy on both datasets. Relative to ETMC, the strongest compared baseline on each dataset, the gains computed from the displayed means are 0.93 percentage point on YouTubeFace and 11.28 percentage points on MVoxCeleb.

\begin{table}[ht]
\centering
\small
\caption{Supplementary accuracy (\%) on larger multi-view datasets.}
\label{tab:appendix-large-datasets}
\setlength{\tabcolsep}{4pt}
\begin{tabular}{lcc}
\toprule
Method & YouTubeFace & MVoxCeleb \\
\midrule
TMC & $84.79\pm0.03$ & $73.55\pm0.08$ \\
ETMC & $86.89\pm0.10$ & $73.58\pm0.04$ \\
RCML & $83.05\pm0.23$ & $69.50\pm0.17$ \\
\rowcolor[RGB]{234,238,234}
TMUR & $\mathbf{87.82\pm0.09}$ & $\mathbf{84.86\pm0.15}$ \\
\bottomrule
\end{tabular}
\end{table}

\subsection{Multimodal Sentiment Classification}
To examine a task beyond the fourteen-dataset benchmark, we also evaluate TMUR on CMU-MOSI, which contains text, audio, and visual modalities. The supplementary evaluation uses the official train/validation/test split and pooled temporal features. Given class imbalance, the reported setting for each method is oriented toward Macro-F1; Table~\ref{tab:appendix-mosi} presents its corresponding test accuracy. TMUR exceeds the reported FUML and ETMC results on both tasks; the absolute gains over FUML are 1.16 and 3.38 percentage points, respectively. These comparisons concern the pooled-feature setting rather than models that explicitly learn temporal interactions.

\begin{table}[ht]
\centering
\small
\caption{Supplementary CMU-MOSI accuracy (\%) using pooled temporal features.}
\label{tab:appendix-mosi}
\setlength{\tabcolsep}{12pt}
\begin{tabular}{lrr}
\toprule
Method & Three-class & Seven-class \\
\midrule
ETMC & 67.00 & 28.80 \\
FUML & 68.02 & 29.62 \\
\rowcolor[RGB]{234,238,234}
TMUR & \textbf{69.18} & \textbf{33.00} \\
\bottomrule
\end{tabular}
\end{table}

\subsection{Computational Cost}
We compare inference cost with representative evidential multi-view methods. Results in Table~\ref{tab:appendix-efficiency} are averaged over fourteen datasets, with each model instantiated using the corresponding view dimensions and class count. Parameters count trainable weights, and FLOPs are forward-pass estimates from \texttt{torch.profiler}. Latency measures inference-only wall-clock time for one batch of 128 samples on an RTX 4090, in evaluation mode with gradient computation disabled. CUDA synchronization is used, with 20 warm-up runs and 80 timed repeats.

\begin{table}[ht]
\centering
\small
\caption{Inference cost averaged over fourteen datasets.}
\label{tab:appendix-efficiency}
\setlength{\tabcolsep}{5pt}
\begin{tabular}{lrrr}
\toprule
Method & Params & FLOPs/sample & Latency/batch \\
       & (M) & (G) & (ms) \\
\midrule
TMC & 3.582 & 0.0072 & 0.944 \\
ETMC & 7.101 & 0.0142 & 1.260 \\
RCML & 3.582 & 0.0072 & 0.240 \\
RCMCL & 0.059 & 0.0001 & 22.704 \\
\rowcolor[RGB]{234,238,234}
TMUR & 2.394 & 0.0092 & 1.673 \\
\bottomrule
\end{tabular}
\end{table}

TMUR has fewer trainable parameters than TMC, ETMC, and RCML. RCMCL has the smallest parameter count and reported FLOP estimate, but the largest measured latency. TMUR's 1.673\,ms latency is higher than those of TMC, ETMC, and RCML, reflecting the additional routing computation. The profiler estimates and measured timings characterize different aspects of cost; a smaller parameter count or FLOP estimate does not necessarily imply lower wall-clock latency.

\subsection{Exploratory Architectural Variants}
We also explored top-2 routed fusion and a shared expert-pool variant. The top-2 experiment showed almost no accuracy change, indicating that concentrating fusion on a small number of selected experts can preserve predictive performance in the tested setting. This observation alone does not establish an inference speedup: the archived sparse implementation incurred additional dispatch overhead and was slower than dense inference.

In the shared-pool variant, experts are not tied to individual views, and aligned samples from different views may select the same experts. This exploratory variant remained competitive with strong existing methods. A systematic evaluation of sparse execution and shared expert pools is left for future work; the main-paper method uses the view-private and collaborative experts described in Section~3.

\subsection{Motivation Analysis via Cross-view Calibration Mismatch}
Figure~\ref{fig:motivation_calibration_appendix} provides additional evidence for the motivation of TMUR using four representative datasets: HandWritten, UCI, NUS, and Caltech-6V. These plots are produced from the saved calibration outputs of TMC and RCML. In each dataset block, the left column shows the uncertainty density of different views, the middle column shows the reliability curves obtained by binning the maximum predicted probability, and the right column shows the reliability curves obtained by binning predictive uncertainty. The first row corresponds to TMC and the second row corresponds to RCML.

Two observations are consistent across these datasets. First, the uncertainty distributions differ substantially from one view to another, even within the same sample population. Second, for both confidence-based and uncertainty-based binning, the same confidence or uncertainty range can lead to clearly different empirical accuracies across views. Therefore, branch-local confidence or uncertainty is not directly comparable across heterogeneous views, and using it as the fusion authority can introduce systematic mismatch. This observation supports the central design choice of TMUR: instead of letting each branch self-assign fusion importance according to its own evidence scale, we use a unified router to infer cross-view fusion weights from the joint multi-view context.

\begin{figure*}[t]
\centering
\begin{minipage}[t]{0.47\textwidth}
\centering
\textbf{HandWritten}\\[2pt]
\includegraphics[width=\linewidth]{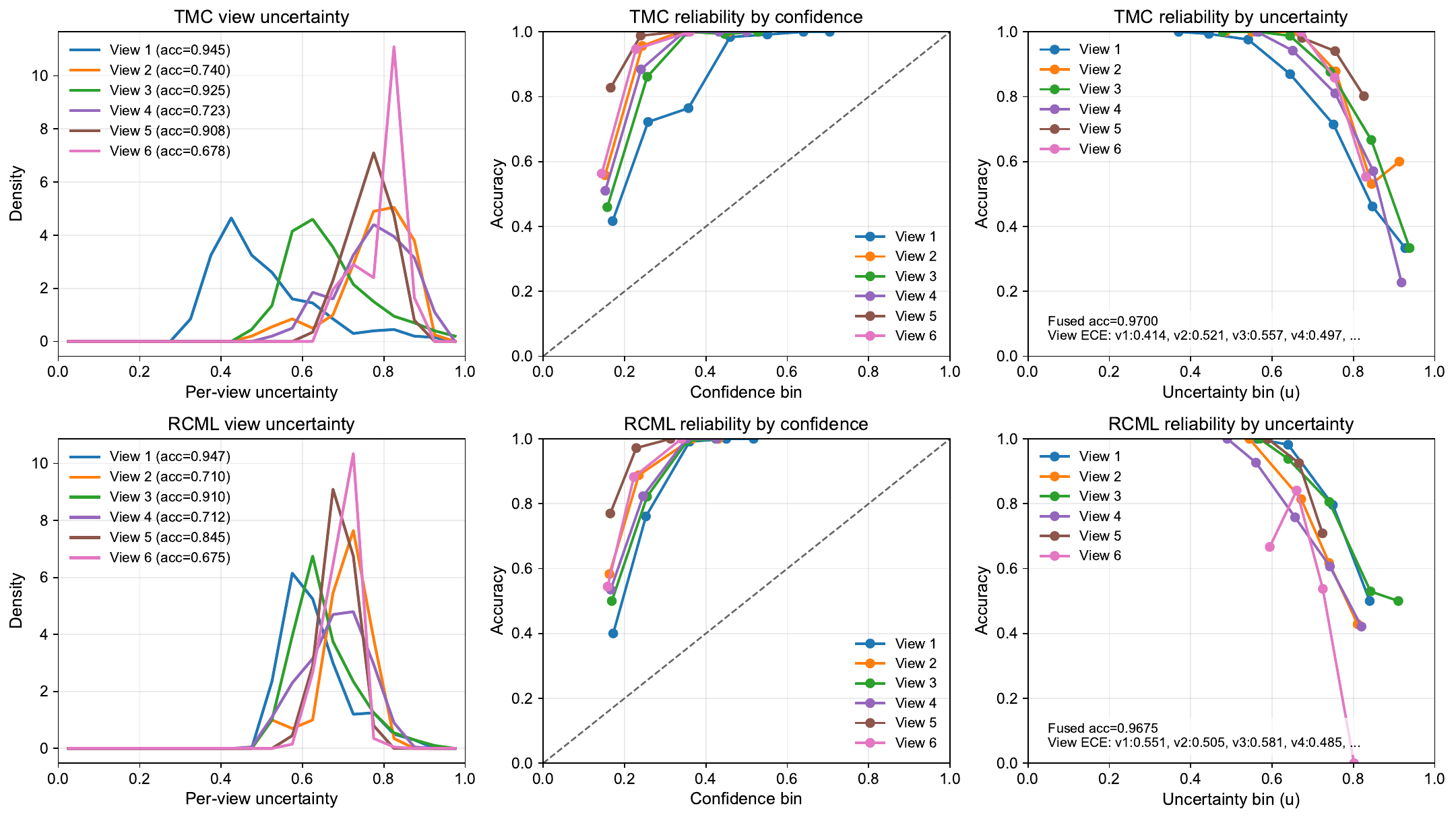}
\end{minipage}\hfill
\begin{minipage}[t]{0.47\textwidth}
\centering
\textbf{UCI}\\[2pt]
\includegraphics[width=\linewidth]{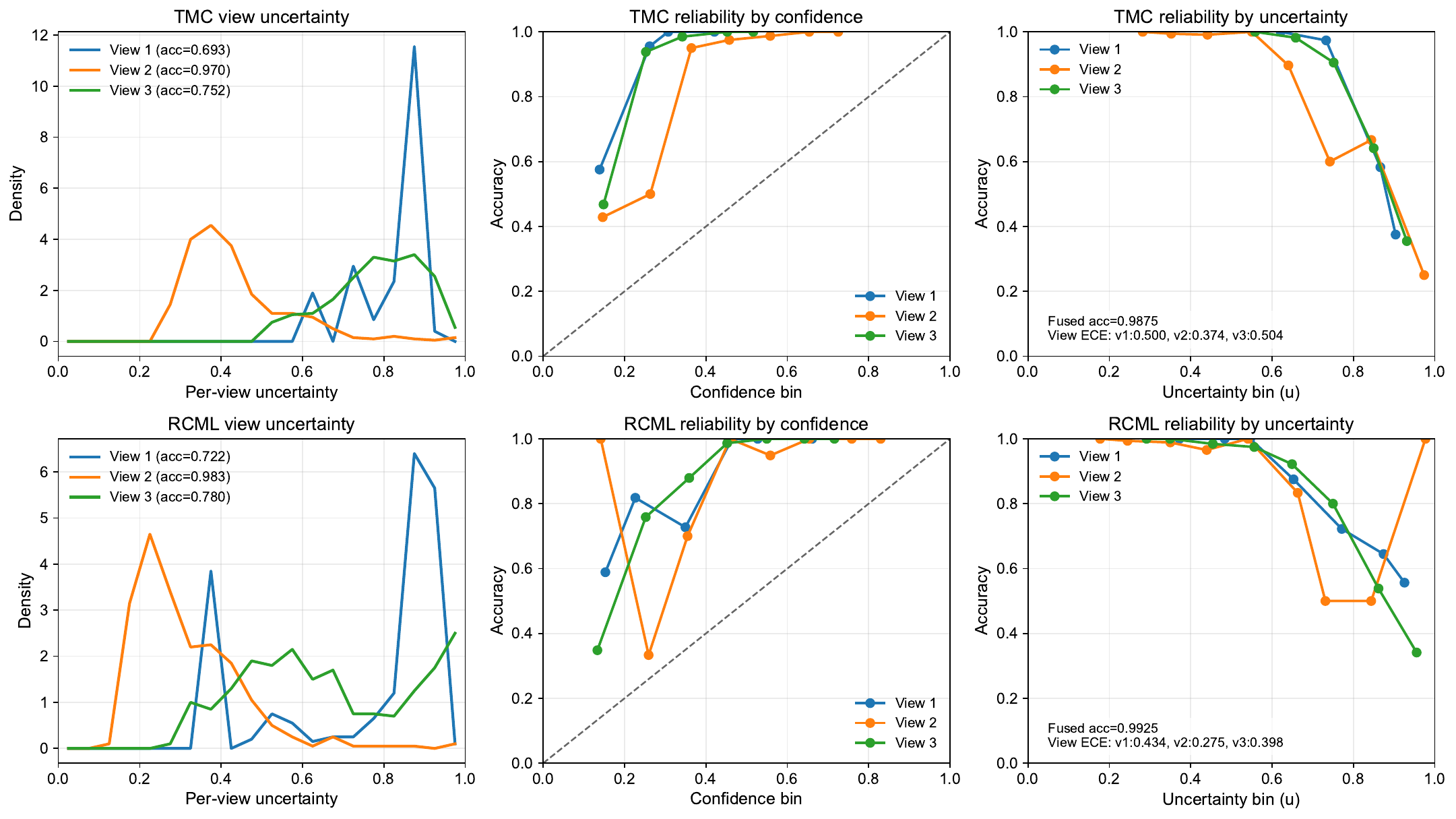}
\end{minipage}

\vspace{0.6em}
\begin{minipage}[t]{0.47\textwidth}
\centering
\textbf{NUS}\\[2pt]
\includegraphics[width=\linewidth]{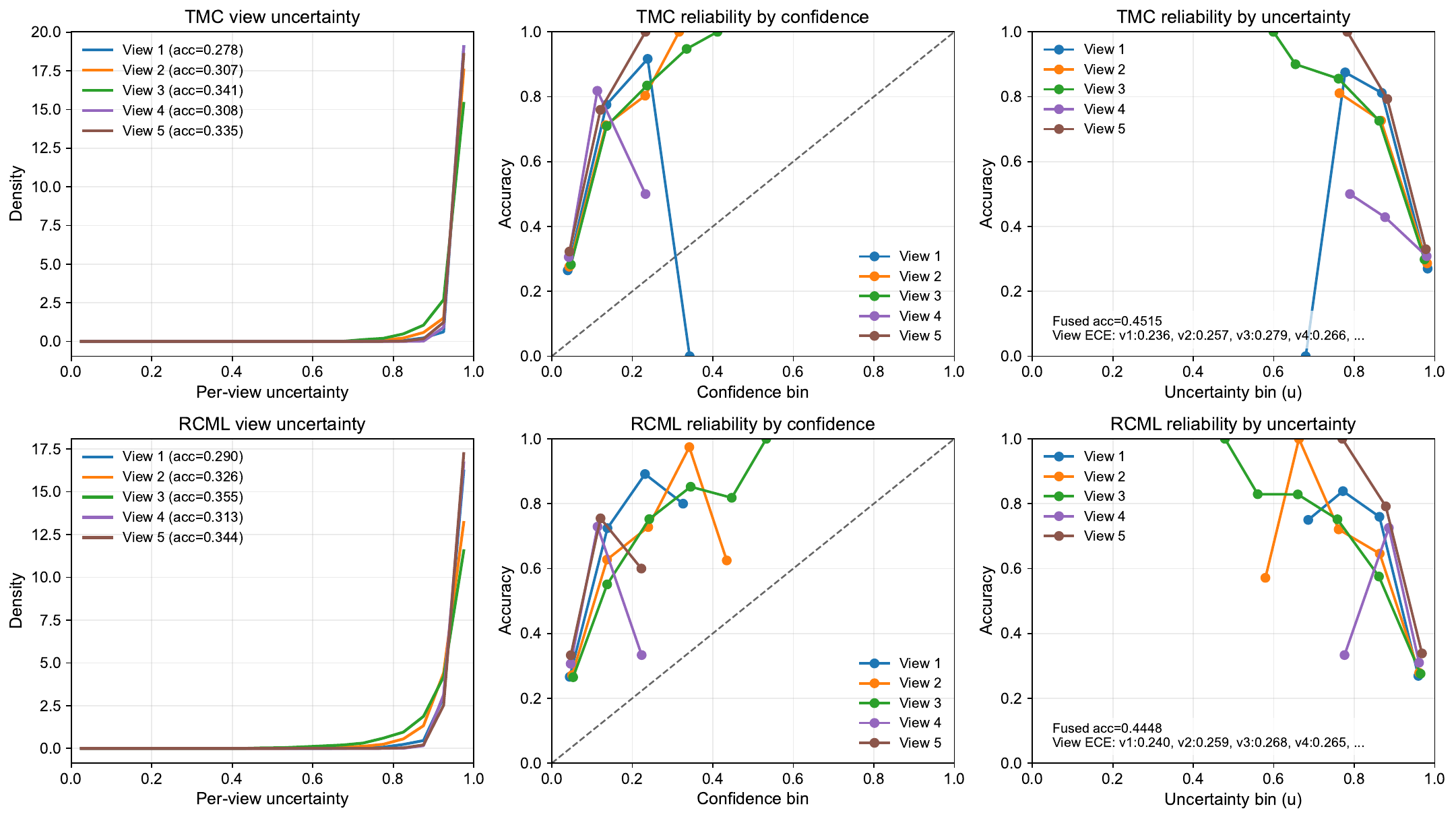}
\end{minipage}\hfill
\begin{minipage}[t]{0.47\textwidth}
\centering
\textbf{Caltech-6V}\\[2pt]
\includegraphics[width=\linewidth]{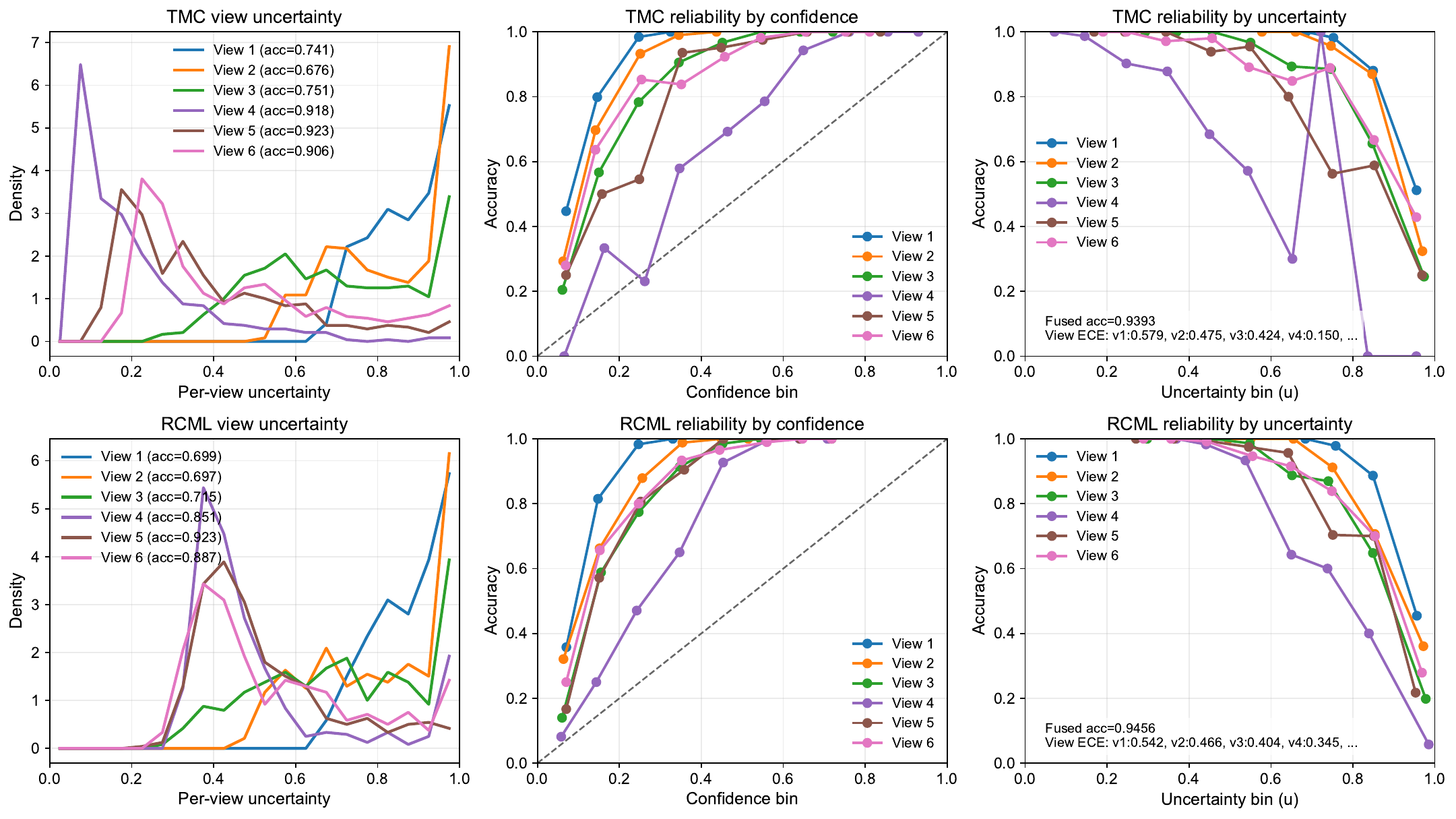}
\end{minipage}

\caption{Motivation analysis using saved calibration results of TMC and RCML on four representative datasets. Within each dataset block, the left column shows per-view uncertainty density distributions, the middle column shows reliability curves based on maximum predicted probability, and the right column shows reliability curves based on predictive uncertainty. The first row is TMC and the second row is RCML. Across datasets, different views exhibit noticeably different calibration behavior, indicating that the same confidence or uncertainty value does not imply the same reliability across views.}
\Description{Four dataset blocks in a two-by-two layout: HandWritten and UCI above NUS and Caltech-6V. Each block preserves three diagnostics for TMC and RCML in two rows and three columns.}
\label{fig:motivation_calibration_appendix}
\end{figure*}

\clearpage
\balance
\section{Supplementary Theoretical Derivations}
\subsection{True-class Probability in Theorem~\ref{thm:scale_bias}}
With the notation of Theorem~\ref{thm:scale_bias}, the remaining true-class probability calculation is
\begin{equation}
\begin{aligned}
p_y(t)&=\frac{1+t r_y}{K+tR},\\
\frac{d\,p_y(t)}{dt}&=\frac{K r_y-R}{(K+tR)^2}.
\end{aligned}
\end{equation}
Thus $r_y>R/K$ implies $d\,p_y(t)/dt>0$, although the support direction is unchanged. The strict uncertainty decrease established in the main paper concerns $R>0$; for the degenerate zero-evidence pattern, $u(t)=1$ is constant.

\subsection{Full Proof of Theorem~\ref{thm:router_advantage}}

We first use the following standard lemma.

\begin{lemma}
For any square-integrable random vector $\bm y$ and statistic $\bm s$, the best $L_2$ predictor among all measurable functions of $\bm s$ is the conditional expectation $\mathbb{E}[\bm y\mid \bm s]$, and the minimum squared error is
\begin{equation}
\inf_f \mathbb{E}\!\left[\|\bm y-f(\bm s)\|_2^2\right]
=
\mathbb{E}\!\left[
\|\bm y-\mathbb{E}[\bm y\mid \bm s]\|_2^2
\right].
\end{equation}
\end{lemma}

\begin{proof}
For any square-integrable measurable $f(\bm s)$,
\begin{equation}
\begin{aligned}
\bm y-f(\bm s)
&=\big(\bm y-\mathbb{E}[\bm y\mid \bm s]\big)\\
&\quad+\big(\mathbb{E}[\bm y\mid \bm s]-f(\bm s)\big).
\end{aligned}
\end{equation}
Expanding the squared norm and taking expectation, the cross term vanishes because
\begin{equation}
\mathbb{E}\!\left[\bm y-\mathbb{E}[\bm y\mid \bm s]\mid \bm s\right]=\bm 0.
\end{equation}
Hence
\begin{equation}
\begin{aligned}
&\mathbb{E}\!\left[\|\bm y-f(\bm s)\|_2^2\right]\\
&=\mathbb{E}\!\left[\|\bm y-\mathbb{E}[\bm y\mid \bm s]\|_2^2\right]\\
&\quad+\mathbb{E}\!\left[\|\mathbb{E}[\bm y\mid \bm s]-f(\bm s)\|_2^2\right],
\end{aligned}
\end{equation}
which is minimized by $f(\bm s)=\mathbb{E}[\bm y\mid \bm s]$.
\end{proof}

Now let $\bm w^\star(x)$ be the oracle routing rule and let $\mathcal{W}_{\mathrm{local}}=\{\bm w(x)=\psi(\bm s(x))\}$ be the class of measurable, simplex-valued branch-local rules. Strong convexity and constrained optimality imply the following bound, including when the oracle lies on the boundary of the simplex. For $0<t<1$, set $\bm w_t=(1-t)\bm w^\star+t\bm w$. Then
\begin{equation}
\begin{aligned}
\mathcal{L}_x(\bm w^\star)&\leq\mathcal{L}_x(\bm w_t)\\
&\leq(1-t)\mathcal{L}_x(\bm w^\star)+t\mathcal{L}_x(\bm w)\\
&\quad-\frac{\mu}{2}t(1-t)\|\bm w-\bm w^\star\|_2^2.
\end{aligned}
\end{equation}
Rearranging, dividing by $t$, and taking $t\downarrow0$ yields
\begin{equation}
\mathcal{L}_x(\bm w(x))-\mathcal{L}_x(\bm w^\star(x))
\ge
\frac{\mu}{2}\|\bm w(x)-\bm w^\star(x)\|_2^2.
\end{equation}
Taking expectation and infimum over $\mathcal{W}_{\mathrm{local}}$ yields
\begin{equation}
\begin{aligned}
&\inf_{\bm w\in\mathcal{W}_{\mathrm{local}}}
\mathbb{E}\!\left[\mathcal{L}_x(\bm w(x))\right]
-
\mathbb{E}\!\left[\mathcal{L}_x(\bm w^\star(x))\right] \\
&\ge
\frac{\mu}{2}
\inf_{\bm w\in\mathcal{W}_{\mathrm{local}}}
\mathbb{E}\!\left[
\|\bm w(x)-\bm w^\star(x)\|_2^2
\right].
\end{aligned}
\label{eq:appendix_router_step1}
\end{equation}

Every rule in $\mathcal{W}_{\mathrm{local}}$ is measurable with respect to $\bm s(x)$. Moreover, the conditional expectation of a simplex-valued random vector remains in the simplex, so the minimizer in the lemma is a feasible local rule. Hence
\begin{equation}
\inf_{\bm w\in\mathcal{W}_{\mathrm{local}}}
\mathbb{E}\!\left[
\|\bm w(x)-\bm w^\star(x)\|_2^2
\right]
=
\mathbb{E}\!\left[
\left\|
\bm w^\star(x)-\mathbb{E}[\bm w^\star(x)\mid \bm s(x)]
\right\|_2^2
\right].
\end{equation}
For the vector-valued oracle, the conditional variance in Theorem~\ref{thm:router_advantage} denotes the trace of its conditional covariance, equivalently $\mathbb{E}[\|\bm w^\star-\mathbb{E}[\bm w^\star\mid\bm s]\|_2^2\mid\bm s]$. Thus,
\begin{equation}
\begin{aligned}
&\inf_{\bm w\in\mathcal{W}_{\mathrm{local}}}
\mathbb{E}\!\left[\mathcal{L}_x(\bm w(x))\right]
-\mathbb{E}\!\left[\mathcal{L}_x(\bm w^\star(x))\right]\\
&\qquad\ge\frac{\mu}{2}\mathbb{E}\!\left[
\mathrm{Var}\!\left(\bm w^\star(x)\mid\bm s(x)\right)\right].
\end{aligned}
\end{equation}
Without an $\bm s$-measurable oracle version, the variance and excess-risk bound are positive. This conditional result does not guarantee oracle attainment by a learned router.

\end{document}